\documentclass{article}

\usepackage[utf8]{inputenc}
\usepackage[T1]{fontenc}
\usepackage{graphicx}
\usepackage{subcaption}
\usepackage{booktabs}
\usepackage{bm}
\usepackage{tcolorbox}
\usepackage{multicol,multirow}
\usepackage{natbib}
\usepackage{comment}
\usepackage{booktabs}

\usepackage{enumerate}   
\usepackage{amsmath}

\usepackage{multirow}
\usepackage{bbm}
\usepackage{tikz}
\def\checkmark{\tikz\fill[scale=0.4](0,.35) -- (.25,0) -- (1,.7) -- (.25,.15) -- cycle;}
\definecolor{mygreen}{rgb}{0,0.6,0}
\usetikzlibrary{positioning,arrows,decorations,decorations.markings,shadows,positioning,arrows.meta,matrix,fit}

\usetikzlibrary{positioning,arrows}
\usepackage{wrapfig}
\DeclareMathOperator*{\argmax}{arg\,max}
\DeclareMathOperator*{\argmin}{arg\,min}

\usepackage{amsmath}
\usepackage{appendix}
\usepackage{amssymb}

\usepackage{amsthm}
\usepackage{hyperref}
\usepackage[capitalise]{cleveref}
\Crefname{assumption}{Assumption}{Assumptions}

\theoremstyle{plain}
\newtheorem{theorem}{Theorem}
\newtheorem{lemma}{Lemma}

\theoremstyle{definition}
\newtheorem{assumption}{Assumption}

\newtheorem{remark}{Remark}

\newcommand{\bG}{\mathbb{G}}
\def\Holder{{H\"{o}lder}}

\def\Cramer{Cram\'{e}r}

\newcommand{\nhs}{n^{\mathrm{hst}}}

\newcommand{\nev}{n^{\mathrm{evl}}}

\newcommand{\para}{\mathrm{para}}
\newcommand{\bN}{\mathcal{N}}
\newcommand{\noo}{N_o}

\newcommand{\cD}{\mathcal{D}}
\newcommand{\cm}{\mathcal{M}}

\newcommand{\bigO}{\mathcal{O}} %

\newcommand{\ts}{\textstyle}

\newcommand{\CR}{\mathrm{CR}}

\newcommand{\hG}{\mathbb{G}}
\newcommand{\Op}{\mathrm{O}_{p}}
\newcommand{\op}{\mathrm{o}_{p}}
\newcommand{\E}{\mathbb{E}}

\newcommand{\pa}{\mathrm{\pa}}

\newcommand{\var}{\mathrm{var}}

\newcommand{\rE}{\mathrm{E}}
\newcommand{\thpol}{\pi^{\theta}}

\newcommand{\Asmse}{\mathrm{Asmse}}
\newcommand{\dml}{\mathrm{DRCS}}
\newcommand{\dm}{\mathrm{DM}}

\newcommand{\braces}[1]{\left\{#1\right\}}
\newcommand{\bracks}[1]{\left[#1\right]}

\newcommand{\Rl}{\mathbb{R}}

\newcommand{\epol}{\pi^\mathrm{e}}
\newcommand{\bpol}{\pi^\mathrm{b}}

\newcommand{\ipw}{\mathrm{IPW}}

\renewcommand{\eqref}[1]{(\ref{#1})}

\newcommand{\RN}[1]{%
  \textup{\uppercase\expandafter{\romannumeral#1}}%
}

\usepackage{color,graphicx,lscape,longtable}

\def\boxit#1{\vbox{\hrule\hbox{\vrule\kern6pt\vbox{\kern6pt#1\kern6pt}\kern6pt\vrule}\hrule}}

\usepackage{xcolor}
\newcount\Comments  
\newcommand{\kibitz}[2]{\ifnum\Comments=1\textcolor{#1}{#2}\fi}

\usepackage{arxiv}
\usepackage{algorithm}
\usepackage{algorithmic}
\usepackage{tikz}
\usetikzlibrary{intersections, calc}

\title{Off-Policy Evaluation and Learning\\
for External Validity under a Covariate Shift}

\newcommand*{\affaddr}[1]{#1} 
\newcommand*{\affmark}[1][*]{\textsuperscript{#1}}
\newcommand*{\equalcontribution}[1][*]{\textsuperscript{*}}
\newcommand*{\email}[1]{\texttt{#1}}

\date{}

\author{%
Masahiro Kato\affmark[1]\footnotemark[1]\thanks{Equal contributions.},\ \ \ \ \ Masatoshi Uehara\affmark[2]\footnotemark[1],\ \ \ \ \ Shota Yasui\affmark[1]\\
\affaddr{\affmark[1]CyberAgent Inc.}\\
\email{masahiro\_kato@cyberagent.co.jp}\\
\email{yasui\_shota@cyberagent.co.jp}\\
\affaddr{\affmark[2] Cornell University}\\
\email{mu223@cornell.edu}\\
}

\begin{document}
\maketitle



\begin{abstract}
We consider evaluating and training a new policy for the evaluation data by using the historical data obtained from a different policy. The goal of \emph{off-policy evaluation} (OPE) is to estimate the expected reward of a new policy over the evaluation data, and that of \emph{off-policy learning} (OPL) is to find a new policy that maximizes the expected reward over the evaluation data. Although the standard OPE and OPL assume the same distribution of covariate between the historical and evaluation data, a covariate shift often exists, i.e., the distribution of the covariate of the historical data is different from that of the evaluation data. In this paper, we derive the efficiency bound of OPE under a covariate shift. Then, we propose doubly robust and efficient estimators for OPE and OPL under a covariate shift by using a nonparametric estimator of the density ratio between the historical and evaluation data distributions. We also discuss other possible estimators and compare their theoretical properties. Finally, we confirm the effectiveness of the proposed estimators through experiments.
\end{abstract}

\section{Introduction}

In various applications, such as ad-design selection, personalized medicine, search engines, and recommendation systems, there is a significant interest in evaluating and learning a new policy from historical data \citep{kdd2009_ads, www2010_cb, AtheySusan2017EPL}. To accomplish this, we use \emph{off-policy evaluation} (OPE) and \emph{off-policy learning} (OPL). The goal of OPE is to evaluate a new policy by estimating the expected reward of the new policy \citep{dudik2011doubly,wang2017optimal,narita2019counterfactual,pmlr-v97-bibaut19a,Kallus2019IntrinsicallyES,Oberst2019}. In contrast, OPL aims to find a new policy that maximizes the expected reward \citep{ZhaoYingqi2012EITR,KitagawaToru2018WSBT,ZhouZhengyuan2018OMPL,Chernozhukov2019}. 

Although an OPE method provides an estimator of the expected reward of a new policy, most existing studies presume that the distribution of covariates is the same between the historical and evaluation data. However, in many real-world applications, the expected reward of a new policy over the distribution of evaluation data is of significant interest, which can be different from the historical data. For example, in the medical literature, it is known that the result of a randomized controlled trial (RCT) cannot be directly transported because the covariate distribution in a target population is different \citep{ColeStephenR.2010GEFR}. This problem is known as a lack of \emph{external validity} 
\citep{PearlJudea2015EVFD}. These situations where historical and evaluation data follow different distributions are also known as \emph{covariate shifts} \citep{shimodaira2000improving,NIPS2007_3248}. This situation is illustrated in Figure~\ref{fig:covariate_shift}.

Under a covariate shift, standard methods of OPE do not yield a consistent estimator of the expected reward over the evaluation data. Moreover, a covariate shift changes the efficiency bound of OPE, which is the lower bound of the asymptotic mean squared error (MSE) among reasonable $\sqrt{n}$-consistent estimators. Besides, standard theoretical analysis of OPE cannot be applied to the covariate shift case as in Remark \ref{rem:standard}.  To handle the covariate shift, we apply importance weighting using the density ratio between the distributions of the covariates of the historical and evaluation data \citep{shimodaira2000improving,Reddi2015}. 

\paragraph{Contributions:} This paper has four main contributions. First, we derive an efficiency bound of OPE under the covariate shift (Section~\ref{sec:semi_lb}). Second, in Section~\ref{sec:pe_csa},
we propose estimators constructed by estimators of the density ratio, behavior policy, and conditional expected reward. Especially, we employ nonparametric density ratio estimation from \citet{kanamori2012kulsif} to estimate the density ratio. The proposed estimator is an efficient estimator, which achieves the efficiency bound under mild nonparametric rate conditions of the estimators of nuisance functions. In addition, this estimator is robust to model-misspecification of estimators in the sense that the resulting estimator is consistent if either (i) models of the density ratio and the behavior policy or (ii) a model of the conditional average treatment effect is correct. Importantly, we do not require the Donsker conditions for those estimators by applying the cross-fitting (Section~\ref{sec:pe_csa}). Third, we propose other possible estimators for our problem setting and compare them (Section~\ref{sec:other_candidates}). Fourth, an algorithm of OPL is proposed based on efficient estimators (Section~\ref{sec:opl}). All proofs are shown in Appendix~\ref{sec:proof}. 

\begin{wrapfigure}{r}{0.5\textwidth}
\begin{minipage}{.44\linewidth}
\scalebox{0.8}{\begin{tikzpicture}[%
>=latex',node distance=2cm, minimum height=0.75cm, minimum width=0.75cm,
state/.style={draw, shape=circle, draw=mygreen, fill=mygreen!10, line width=0.5pt},
action/.style={draw, shape=rectangle, draw=red, fill=red!10, line width=0.5pt},
reward/.style={draw, shape=rectangle, draw=blue, fill=blue!10, line width=0.5pt}
]
\node[state] (S0) at (0,0) {$X \sim p(x)$};
\node[action,right of=S0] (A0) {$A$};
\node[reward,below of=A0] (R0) {$Y$};
\draw[->] (S0) -- (A0) node[pos=.45,above] {$\bpol$};
\draw[->] (S0) -- (R0);
\draw[->] (A0) -- (R0);
\end{tikzpicture}}
  \caption*{Historical data}
  \label{fig:sub-first}
\end{minipage}
\begin{minipage}{.47\linewidth}
\scalebox{0.8}{\begin{tikzpicture}[%
>=latex',node distance=2cm, minimum height=0.75cm, minimum width=0.75cm,
state/.style={draw, shape=circle, draw=green, fill=green!10, line width=0.5pt},
action/.style={draw, shape=rectangle, draw=gray, fill=gray!10, line width=0.5pt},
reward/.style={draw, shape=rectangle, draw=gray, fill=gray!10, line width=0.5pt}
]
\node[state] (S0) at (0,0) {$X\sim q(x)$};
\node[action,right of=S0] (A0) {$A$};
\node[reward,below of=A0] (R0) {$Y$};
\draw[->] (S0) -- (A0) node[pos=.5,above] {$\epol$};
\draw[->] (S0) -- (R0);
\draw[->] (A0) -- (R0);

\end{tikzpicture}}
 \caption*{Evaluation data}
  \label{fig:sub-first2}
\end{minipage}
  \caption{OPE under a covariate shift. Covariate, action, reward are denoted by $X,A,Y$. Evaluation and behavior policies are denoted by $\epol,\bpol$. Here, $p(x)\neq q(x)$, and the density ratio $q(x)/p(x)$ is unknown. The density $p(y\mid a,x)$ is the same in historical and evaluation data. For the evaluation data, $A$ and $Y$ are not observed.}
  \label{fig:covariate_shift}
  \vspace{-1cm}
\end{wrapfigure}

\paragraph{Related work:}  
The difference between distributions of covariates conditioned on a chosen action is also called a covariate shift \citep{pmlr-v28-zhang13d,pmlr-v48-johansson16}. A covariate shift in this paper refers to the different distributions of covariates between historical and evaluation data. \citet{DahabrehIssaJ.2019Gcif,JohanssonFredrik2018LWRf,pmlr-v108-sondhi20a} analyzed the treatment effect estimation under a covariate shift; however, our perspective and analysis are completely different from theirs. Besides, there are many studies regarding the external validity on a 
causal directed acyclic graph \citep{PearlJ2011ToCa,PearlJudea2015EVFD}. 
This paper focuses on statistical inference and learning instead of an identification strategy. 

\section{Problem Formulation}
\label{sec:prob_for}
In this section, we introduce our problem setting and review existing literature. 
\subsection{Data-Generating Process with Evaluation Data}
Let $A_i$ be an action taking variable in $\mathcal{A}$ and $Y_i\in\mathbb{R}$ be a reward of an individual $i\in\mathbb{N}$. Let $X_i$ and $Z_i$ be the \emph{covariate} observed by the decision maker when choosing an action, and $\mathcal{X}$ be the space of the covariate. Let a policy $\pi:\mathcal{X}\times\mathcal{A}\to[0,1]$ be a function of a covariate $x$ and action $a$, which can be considered as the probability of choosing an action $a$ given $x$. In this paper, we have access to \emph{historical} and \emph{evaluation data}. For the historical data, we can observe a dataset $\mathcal{D}^{\mathrm{hst}}=\{(X_i, A_i, Y_i)\}^{n^{\mathrm{hst}}}_{i=1}$, which are \emph{independent and identically distributed} (i.i.d.)for the evaluation data, we can observe an i.i.d. dataset $\mathcal{D}^{\mathrm{evl}}=\{Z_i\}^{n^{\mathrm{evl}}}_{i=1}$, where $n^{\mathrm{hst}}$ and $n^{\mathrm{evl}}$ denote the sample sizes of historical and evaluation data, respectively. We assume $\mathcal{D}^{\mathrm{hst}}$ and $\mathcal{D}^{\mathrm{evl}}$ are independent. Then, the data-generating process (DGP) is defined as follows:
\begin{align*}\ts
&\mathcal{D}^{\mathrm{hst}}=\{(X_i, A_i, Y_i)\}^{n^{\mathrm{hst}}}_{i=1}\sim p(x)\pi^{\mathrm{b}}(a \mid x)p(y \mid x,a),\ \ \ \mathcal{D}^{\mathrm{evl}}=\{Z_i\}^{n^{\mathrm{evl}}}_{i=1}\sim q(z),
\end{align*}
where $n^{\mathrm{hst}}=\rho n$, $n^{\mathrm{evl}}=(1-\rho)n$, $p(x)$ and $q(x)$ are densities\footnote{We use $x$ and $z$ exchangeably noting the spaces of $X$ and $Z$ are the same such as $q(x),q(z)$ and $p(x),p(z)$. On the other hand, we strictly distinguish $X_i$ and $Z_i$ noting these are different random variables.} over $\mathcal{X}$, and $\rho\in(0,1)$ is a constant. The policy $\pi^{\mathrm{b}}(a \mid x)$ of the historical data is called a \emph{behavior policy}. We generally assume $p(x),\,q(x)$ and $\bpol(a\mid x)$ to be unknown. In comparison to the usual situation of OPE, the density of historical data, $p(x)$, can be different from that of the evaluation data, $q(x)$. 

\paragraph{Notation:} This paper distinguishes the covariates between the historical and evaluation data as $X_i$ and $Z_i$, respectively. Hence, for a function $\mu:\mathcal{X}\to\mathbb{R}$, $\E[\mu(X)]$ and $\E[\mu(Z)]$ imply taking expectation over historical and evaluation data, respectively. Likewise, the empirical approximation is denoted as $\E_{\nhs}[\mu(X)]=1/\nhs\sum_{i}\mu(X_i)$ and $\E_{\nev}[\mu(X)]=1/\nev\sum_{i}\mu(Z_i)$. Additionally, let $\|\mu(X,A,Y)\|_2$ be $\E[\mu^2(X,A,Y)]^{1/2}$ for the function $\mu$,  $\E_{p(x,a,y)}[\mu(x,a,y)]$ be $\int \mu(x,a,y)p(x,a,y)\mathrm{d}(x,a,y)$, the asymptotic MSE of estimator $\hat R$ be $\Asmse[\hat R]=\lim_{n\to \infty}n \E[(\hat R-R)^2]$, and $\mathcal{N}(0,A)$ be a normal distribution with mean $0$ and variance $A$. Besides, we use functions $r(x)=q(x)/p(x)$, $w(a,x)=\epol(a\mid x)/\bpol(a\mid x)$, and $f(a,x)=\E[Y\mid X=x,A=a]$. Let us denote the estimators of $r(x)$, $w(a,x)$, and $f(a,x)$ as $\hat r(x)$, $\hat w(a, x)$, and $\hat f(a, x)$, respectively. Other notations are summarized in Appendix~\ref{sec:notation}.

\begin{remark}
Although we do not explicitly use counter-factual notation \citep{rubin87}, if we assume the usual conditions, our results immediately apply (Appendix~\ref{sec:idenfication}). 
\end{remark}

\subsection{Off-Policy Evaluation and Learning}
\label{sec:opeopl}

We are interested in estimating the expected reward of an \emph{evaluation policy} $\pi^{\mathrm{e}}(a \mid x)$, which is pre-specified for the evaluation data. Here, we assume a \emph{covariate shift}, which is a common situation in the literature of external validity. Under a covariate shift, while the conditional distribution of $y$ are the same between historical and evaluation data, the distribution of evaluation data is different from historical data, i.e., the distribution of evaluation data with evaluation policy $\epol$ follows $q(z)\epol(a \mid z)p(y\mid a,z)$. Then, we define the expected reward of evaluation policy as follows:
\begin{align}\ts
\label{def:policy_value}
R(\epol) := \mathbb{E}_{q(z)\epol(a \mid z)p(y \mid a,z)}\left[y\right].
\end{align}
Then, the first goal is OPE; i.e., estimating $R(\pi^{\mathrm{e}})$ using the historical data $\{X_i,A_i,Y_i\}_{i=1}^{\nhs}$ and evaluation data $\{Z_i\}^{\nev}_{i=1}$. The second goal is OPL; i.e., training a new policy that maximizes the expected reward as $\pi^* = \argmax_{\pi\in\Pi}R(\pi)$, where $\Pi$ is the policy class. In some cases, to construct an estimator $R(\pi)$, we use $r(x)$, $w(a,x)$, and $f(a,x)$. These functions are called \emph{nuisance functions}. Let $\hat{r}(x)$, $\hat{w}(a,x)$, and $\hat{f}(a,x)$ be their estimators.

\paragraph{Assumptions:} 
We assume strong overlaps for $r(x)$, $w(a, x)$ and theirs estimators and boundedness for $Y_i$ and $\hat{f}$ using a constant $R_{\max} > 0$.
\begin{assumption}
\label{asm:global}
$0\leq r(x) \leq C_1,\,0\leq w(a,x)\leq C_2,\,0\leq  Y_i \leq R_{\max}$.
\end{assumption}

\begin{assumption}
\label{asm:global2}
$0\leq \hat r(x) \leq C_1,\,0\leq \hat w(a,x)\leq C_2,\,0\leq \hat f(a,x) \leq R_{\max}$.
\end{assumption}

\subsection{Preliminaries}
\label{sec:pre}
Here, we review the existing works of OPE, OPL, and the density ratio estimation.  

\paragraph{Standard OPE and OPL:} We review three types of standard estimators of $\E_{p(x)\epol(a\mid x)p(y\mid x,a)}[y]$ under the case where $q(x)=p(x)$ in (\ref{def:policy_value}). The first estimator is an inverse probability weighting (IPW) estimator given by $\E_{\nhs}[\hat w(A,X)Y]$ \citep{Horvitz1952,rubin87,cheng1994,hirano2003efficient,swaminathan15a}. Even though this estimator is unbiased when the behavior policy is known, it often suffers from high variance. The second estimator is a direct method (DM) estimator $\E_{\nhs}[\hat f(A,X)]$ \citep{HahnJinyong1998OtRo}, which is weak against model misspecification for $f(a,x)$. The third estimator is a doubly robust estimator \citep{robins94} defined as
\begin{align}\ts
\label{eq:doubl_conventional}
    \E_{\nhs}[\hat w(A,X)\{Y-\hat f(A,X)\}+\E_{\epol(a\mid X)}[\hat f(a,X) \mid X]]. 
\end{align}
Under certain conditions, it is known that this estimator achieves the efficiency bound (a.k.a semiparametric lower bound), which is the lower bound of the asymptotic MSE of OPE, among regular $\sqrt{n}$-consistent estimators \citep[Theorem 25.20]{VaartA.W.vander1998As} \footnote{Formally, regular estimators means estimators whose limiting distribution is insensitive to local changes to the DGP. Refer to \citet[Chapter 7]{VaartA.W.vander1998As} }. This efficiency bound is
\begin{align}\ts 
\label{eq:bound}
\E[w^2(A,X)\var[Y\mid A,X]]+\var[v(X)], 
\end{align}
where $v(x)=\E_{\epol(a\mid x)}[f(a,x)\mid x]$ \citep{narita2019counterfactual}. Such estimator is called an \emph{efficient estimator}. These estimators are also used for OPL \citep{ZhangBaqun2013Reoo,AtheySusan2017EPL}.

\paragraph{Density Ratio Estimation:}
To estimate $R(\pi)$, we apply an importance weighting using the density ratio between distributions of historical and evaluation covariates. For example, if we know $r(x)$ and $w(a,x)$, we can construct an estimator of $R(\epol)$ as $\E_{\nhs}[r(X)w(A,X)Y]$. If we know the behavior policy as in an RCT, we can exactly know $w(a,x)$. However, since we do not know the density ratio $r(x)$ directly even in an RCT, we have to estimate $r(x)$ using the covariate data: $\{X_i\}_{i=1}^{\nhs}$ and $\{Z_i\}_{i=1}^{\nev}$. To estimate the density ratio $r(x)$, we use a nonparametric one-step loss based estimator. For example, we employ \emph{Least-Squares Importance Fitting} (LSIF), which uses the squared loss to fit the density-ratio function \citep{kanamori2012kulsif}. We show details in Appendix~\ref{appdx:uLSIF}.

\remark[Difference from standard OPE problems] \label{rem:standard}


Our current problem, i.e., policy evaluation \emph{under a shift in domain and policy}, differs from a standard policy evaluation problem \emph{only under a shift in the policy}. For our domain and policy shift problem, we assume a stratified sampling, i.e, fixed $\rho$ w.r.t $n$. Instead, in the literature of a policy shift, people assume a sampling scheme is i.i.d. As in \citet{WooldridgeJeffreyM.2001APOW}, the difference of these two sampling schemes makes the analysis different. 

We can also assume that samples are i.i.d in our problem by treating $\rho$ is a random variable and assuming each replication follows a \emph{mixture distribution} \citep{DahabrehIssaJ.2019Gcif}. However, under this assumption, the efficiency bound cannot be calculated in an explicit form. Besides, $\rho$ is often given as a constant value by some design \citep{QinJing1998IfCa}. 

\section{Efficiency Bound under a Covariate Shift}
\label{sec:semi_lb}
We discuss the efficiency bound of OPE under a covariate shift. Efficiency bound is defined for an estimand under some posited models of the DGP \citep{bickel98}. If this posited model is a parametric model, it is equal to the \Cramer-Rao lower bound. When this posited model is non or semiparametric model, we can still define a corresponding \Cramer-Rao lower bound. In this paper, we modify the standard theory under i.i.d. sampling to the current problem with a stratified sampling scheme. The formal definition is shown in Appendix~\ref{sec:semi_ld}. 

Here, we show the efficiency bound of OPE under a covariate shift. 
\begin{theorem}
\label{thm:efficiency}
The efficiency bound of $R(\epol)$ under fully nonparametric models  is 
\begin{align}
\label{eq:bound2}
     \Upsilon(\epol)=\rho^{-1}\E[r^2(X)w^2(A,X)\var[Y\mid A,X]]+(1-\rho)^{-1}\var[v(Z)],
\end{align}
where $v(z)=\E_{\epol(a\mid z)}[f(a,z)\mid z]$. The efficiency bound under a nonparametric model with fixed $p(x)$ and $\bpol(a \mid x)$ is the same. 
\end{theorem}

Three things are remarked. First, the knowledge of the density function of the historical data $p(x)$ and the behavior policy $\bpol(a\mid x)$ does not change the efficiency bound \eqref{eq:bound}. 
This is because the target functional does not include these two densities.
Second, the efficiency bound under a covariate shift \eqref{eq:bound2} reduces to the bound without a covariate shift \eqref{eq:bound} in a special case, $r(x)=1$ and $\rho=0.5$. 
Then, we can see \eqref{eq:bound2}$=2\times $\eqref{eq:bound}. The factor $2$ originates from the scaling of asymptotic MSE. Third, we need to calculate the \emph{efficient influence function}, which is a key function to derive the efficiency bound. This function is useful to construct the efficient estimator.

\section{OPE under a Covariate Shift}
\label{sec:pe_csa}
For OPE under a covariate shift, we propose an estimator constructed from the following basic form:
\begin{align}\ts
\label{eq:double}
    \E_{\nhs}[\hat r(X)\hat w(A,X)\{Y-\hat f(A,X)\}]+\E_{\nev}[\hat v(Z)],
\end{align}
where $\hat r(x)$, $\hat w(a,x)$, and $\hat f(a,x)$ are nuisance estimators of $r(x)$,  $w(a,x)$, and $f(a,x)$, and $\hat v(z)=\E_{\epol(a\mid z)}[\hat f(a,z)\mid z]$. As well as the standard doubly robust estimator \eqref{eq:doubl_conventional}, the above form is designed to have the double robust structure regarding the model specifications of $r(x)w(a, x)$ and $f(a, x)$. First, we consider the case where $\hat r(x)=r(x)$ and $\hat w(a,x)=w(a,x)$, but $\hat f(a,x)$ is equal to $f^{\dagger}(a,x)$ and different from $f(a,x)$, i.e., we have correct models for $r(x)$ and $w(a,x)$, but not for $f(a,x)$. Then, \eqref{eq:double} is a consistent estimator for $R(\epol)$ since
\begin{align*}\ts
     &\mathbb{E}_{\nhs}[r(X)w(A,X)Y]+\mathbb{E}_{\nev}[\E_{\epol(a \mid Z)}[f^{\dagger}(a,Z)\mid Z]]-\mathbb{E}_{\nhs}[r(X)w(A,X)f^{\dagger}(A,X)]\\
     &\approx \mathbb{E}_{\nhs}[r(X)w(A,X)Y]+0
      \approx R(\epol). 
\end{align*}
Second, we consider the case where $\hat f(a,x)=f(a,x)$, but $\hat r(x)$ and $\hat w(a,x)$ are equal to functions $r^{\dagger}(x)$ and $w^{\dagger}(a,x)$, which are different from $r(x)$ and $w(a,x)$, respectively, i.e, we have correct models for $f(a,x)$, but not for $r(x)$ and $w(a,x)$. Then, \eqref{eq:double} is a consistent estimator for $R(\epol)$ since
\begin{align*}\ts
     &\mathbb{E}_{\nhs}[r^{\dagger}(X)w^{\dagger}(a,x)\{Y-f(A,X) \}]+\mathbb{E}_{\mathrm{n^{evl}}}[\E_{\epol(a \mid Z)}[f(a,Z)\mid Z]]\\
     &\approx \mathbb{E}_{\nev}[\E_{\epol(a \mid Z)}[f(a,Z)\mid Z]]+0
      \approx R(\epol). 
\end{align*}
The formal result is given later in Theorem~\ref{thm:model}. 

\begin{algorithm}[tb]
   \caption{\small{Doubly Robust Estimator under a Covariate Shift}}
   \label{alg:dml}
\begin{algorithmic}
    \STATE \textbf{Input}: The evaluation policy $\epol$.
    \STATE Take a $\xi$-fold random partition $(I_k)^{\xi}_{k=1}$ of observation indices $[\nhs] = \{1,\dots,\nhs\}$ such that the size of each fold $I_k$ is $\nhs_k=\nhs/\xi$.
    \STATE Take a $\xi$-fold random partition $(J_k)^{\xi}_{k=1}$ of observation indices $[\nev] = \{1,\dots,\nev\}$ such that the size of each fold $J_k$ is $\nev_k=\nev/\xi$.
    \STATE For each $k\in[\xi]=\{1,\dots,\xi\}$, define $I^c_k:=\{1,\dots,\nhs\}\setminus I_k$ and $J^c_k:=\{1,\dots,\nev\}\setminus J_k$.
    \STATE Define $(\mathcal{S}_k)^{\xi}_{k=1}$ with $\mathcal{S}_k = \{\{(X_i, A_i, Y_i)\}_{i\in I^c_k}, \{Z_j\}_{j\in J^c_k}\}$.
    \FOR{$k\in[\xi]$}
    \STATE Construct estimators $\hat w_k(a, x)$, $\hat r_k(x)$, and $\hat f_k(a,x)$ using $\mathcal{S}_k$.
    \STATE Construct an estimator $\hat{R}_k$ defined as \eqref{eq:case2}.
    \ENDFOR
    \STATE Construct an estimator $\hat R$ of $R$ by taking the average of $\hat{R}_k$ for $k\in[\xi]$, i.e., 
    $\hat R = \frac{1}{\xi}\sum^{\xi}_{k=1}\hat{R}_k$.
\end{algorithmic}
\end{algorithm}  

Next, we consider estimating $r(x)$,  $w(a,x)$, and $f(a,x)$. For example, for $f(a,x)$ and $w(a,x)$, we can apply complex and data-adaptive regression and density estimation methods such as random forests, neural networks, and highly adaptive Lasso \citep{DiazIvan2019Mlit}. Note that $\hat w(a,x)$ is estimated as $\epol/\hat{\pi}^{b}$ since $\epol$ is known, where $\hat \pi^{b}$ is an estimator of $\bpol$. For $r(x)$, we can use the data-adaptive density ratio method in Section~\ref{sec:pre}. Although 
such complex estimators approximate the true values well, it is pointed out that such estimators often violate the Donsker condition \citep{VaartA.W.vander1998As,ChernozhukovVictor2018Dmlf}. 
\footnote{When the square integrable envelope function exists and the metric entropy of the function class is controlled at some rates, Donker's condition is satisfied \citep[Chapter 19]{VaartA.W.vander1998As}.}, which is required to obtain the asymptotic distribution of an estimator of interest, such as (\ref{eq:double}). 

For deriving the asymptotic distributions of an estimator of $R(\epol)$ using estimators without the Donsker condition, we apply cross-fitting \citep{klaassen1987,ZhengWenjing2011CTME,ChernozhukovVictor2018Dmlf} based on \eqref{eq:double}. 
The procedure is as follows. First, we separate data $\mathcal{D}^{\mathrm{hst}}$ and $\mathcal{D}^{\mathrm{evl}}$ into $\xi$ groups. Next, using samples in each group, we estimate the nuisance functions nonparametrically. Then, we construct an estimator of $R(\epol)$ using the nuisance estimators. For each group $k\in\{1,2,\dots,\xi\}$, we define 
\begin{align}\ts
\label{eq:case2}
   \hat R_k = &{\mathbb{E}}_{n^\mathrm{hst}_{k}}[\hat r^{(k)}(X)\hat w^{(k)}(A,X)\{Y-\hat f^{(k)}(A,X) \}]+{\mathbb{E}}_{n^\mathrm{evl}_{k}}[\mathbb{E}_{\epol}[\hat f^{(k)}(a,Z)|Z]],
\end{align} 
where ${\mathbb{E}}_{n^\mathrm{hst}_{k}}$ is the sample average over $k$-th partitioned historical data with $n^\mathrm{hst}_{k}$ samples and ${\mathbb{E}}_{n^\mathrm{evl}_{k}}$ is the sample average over $k$-th partitioned evaluation data with $n^\mathrm{evl}_{k}$ samples. Finally, we construct an estimator of $R(\epol)$ by taking the average of the $K$ estimators, $\{\hat R_k\}$. We call the estimator \emph{doubly robust estimator under a covariate shift} (DRCS) and denote it as $\hat{R}_{\mathrm{DRCS}}(\epol)$. The whole procedure is given in Algorithm~\ref{alg:dml}.

In the following, we show the asymptotic property of 
$\hat{R}_{\mathrm{DRCS}}(\epol)$. First, $\hat{R}_{\mathrm{DRCS}}(\epol)$ is efficient.

\begin{theorem}[Efficiency]
\label{thm:main2}
For $k\in \{1,\cdots,\xi\}$, assume $\alpha\beta=\op(n^{-1/2}),\alpha=\op(1),\beta=\op(1)$ where $\|\hat r^{(k)}(X)\hat w^{(k)}(A,X)-r(X)w(A,X)\|_2=\alpha,\|\hat f^{(k)}(A,X)-f(A,X)\|_2=\beta$. Then, $\sqrt{n}(\hat{R}_{\mathrm{DRCS}}(\epol)-R(\epol))\stackrel{d}{\rightarrow}\mathcal{N}(0,\Upsilon(\epol))$, where  $\Upsilon(\epol)$ is the efficiency bound in Theorem \ref{thm:efficiency}. 
\end{theorem}

Importantly, the Donsker condition is \emph{not} needed for nuisance estimators owing to the cross-fitting and the doubly robust form of $\hat{R}_{\mathrm{DRCS}}$. What we only need are rate conditions. The rate conditions are mild since these are nonparametric rates smaller than $1/2$. For example, this is satisfied when $p=1/4$, $q=1/4$. With some smoothness conditions, the nonparametric estimator $\hat f(a,x)$ can achieve this convergence rate \citep{WainwrightMartinJ2019HS:A}. Regarding $r(x)w(a,x)$, we can show that if $\hat r(x)$ and $\hat w(a,x)$ similarly satisfy some nonparametirc rates, $\hat r(x)\hat w(a,x)$ satisfies it as well. 
\begin{lemma}
\label{lem:den}
Assume $\|\hat r(X)-r(X)\|_2 =\op(n^{-p})$ and $\|\hat w(A,X)-w(A,X)\|_2 = \op(n^{-p})$. Then, $\|\hat r(X)\hat w(A,X)-r(X)w(A,X)\|_2=\op(n^{-p})$.
\end{lemma}

Next, we formally show double robustness of the estimator, i.e., the estimator is consistent if either $r(x)w(a,x)$ or $f(a,x)$ is correct. 
\begin{theorem}[Double robustness]
\label{thm:model}
For $k\in \{1,\cdots,\xi\}$, assume $\exists\,f^{\dagger},r^{\dagger},w^{\dagger}$,  $\|\hat f^{(k)}(A,X)-f^{\dagger}(A,X)\|_2 =\op(1),\ \ \ \|\hat r^{(k)}(X)\hat w^{(k)}(A,X)-r^{\dagger}(X)w^{\dagger}(A,X)\|_2 =\op(1).$ If $r^{\dagger}(x)w^{\dagger}(a,x)=r(x)w(a,x)$ or $q^{\dagger}(a,x)=q(a,x)$ holds, the estimator $\hat{R}_{\mathrm{DRCS}}(\epol)$ is consistent. 
\end{theorem}
 In a standard OPE, the DR type estimator is consistent when we know a behavior policy. In contrast, under a covariate shift, even when the behavior policy is known, we cannot claim $\hat{R}_{\mathrm{DRCS}}(\epol)$ is consistent since $r(x)$ is unknown. This result suggests the estimation of $r(x)$ is crucial.

\begin{remark}[OPE with Known Distribution of Evaluation Data]
As a special case of OPE under a covariate shift, we consider a case where $q(x)$ is known. This case can be regarded as a standard OPE situation by regarding $p(x)\epol(a\mid x)$ as the behavior policy, the evaluation policy as $q(x)\epol(a\mid x)$, and $(A,X)$ as the action. The details of this setting is shown in Appendix~\ref{sec:known}
\end{remark}

\begin{remark}[Relation with \citet{PearlJudea2015EVFD}]
A transport formula \citep[(3.1)]{PearlJudea2015EVFD} essentially leads to the DM estimator $\E_{\nev}[\hat v(Z)]$. Though they propose a general identification strategy, they do not discuss how to conduct efficient estimation given finite samples. 
\end{remark}

\begin{remark}[Construction of $\hat{R}_{\mathrm{DRCS}}(\epol)$] We construct  $\hat{R}_{\mathrm{DRCS}}(\epol)$ so that it has a doubly robust structure. The construction is also motivated by the efficient influence function. More specifically, this estimator is introduced by plugging the nuisance estimates into the efficient influence function.
\end{remark}

\section{Other Candidates of Estimators}
\label{sec:other_candidates}
We have discussed the doubly robust estimator in the previous section. Next, we propose other estimators under a covariate shift based on IPW and DM estimators. We analyze the property of each estimator with nuisance estimators obtained from the classical kernel regression \citep{Nada:1964,Wats:1964}. We show regularity conditions and formal results of Theorems~\ref{thm:ipw2}--\ref{thm:dm} in \cref{sec:proof}.

\subsection{IPW Estimators and DM Estimator}
We consider IPW and DM type estimators under a covariate shift for \emph{each case} where we have an oracle of $\bpol(a\mid x)$ and we do not have any oracles of nuisance functions, \emph{respectively}. In comparison to a standard OPE case, we can consider two fundamentally different IPW type estimators. 

\paragraph{IPW estimator with oracle $\bpol(x)$:}  
This is a natural setting in an RCT and and A/B testing since we assign actions following a certain probability in theses cases. Let us define an IPW estimator under a covariate shift with the true behavior policy $\bpol(a\mid x)$ (IPWCSB) as $\hat R_{\mathrm{IPWCSB}}(\epol)=\E_{\nhs}\left[\frac{\hat q(X)}{\hat p(X)}\frac{\epol(A\mid X)Y}{\bpol(A\mid X)}\right]$. For example, we use classical kernel density estimators of $q(x)$ and $p(x)$:
\begin{align}\ts
\label{eq:kernel}
\hat q_h(x)=\frac{1}{\nev}\sum_{i=1}^{\nev}h^{-d}K\left(\frac{Z_i-x}{h^d}\right),\, \hat p_h(x)=\frac{1}{\nhs}\sum_{i=1}^{\nhs}h^{-d}K\left(\frac{X_i-x}{h^d}\right), 
\end{align}
where $K(\cdot)$ is a kernel function, $h$ is the bandwidth of $K(\cdot)$, and $d$ is a dimension of $x$. 
When using a kernel estimator, we obtain the following theorem.

\begin{theorem}[Informal]
\label{thm:ipw2}
When $\hat q(x)=\hat q_h(x),\,\hat p(x)=\hat p_h(x)$, the asymptotic MSE of $\hat R_{\mathrm{IPWCSB}}(\epol)$ is $\rho^{-1}\var[r(X)\{w(A,X)Y-v(X)\}]+(1-\rho)^{-1}\var[v(Z)]$.
\end{theorem}

\paragraph{Fully nonparametric IPW estimator:} Next, \emph{for the case without the oracle $\bpol$}, let us define an IPW estimator under a covariate shift (IPWCS) as $\hat R_{\mathrm{IPWCS}}(\epol)=\E_{\nhs}\left[\frac{\hat q(X)\epol(A \mid X)Y}{\hat p(X)\hat \pi^\mathrm{b}(A \mid X)}\right]$. This estimator achieves the efficiency bound. 
\begin{theorem}[Informal]
\label{thm:ipw3}
When $\hat q(x)=\hat q_h(x),\,\hat p(x)=\hat p_h(x)$ and $\hat \pi^b(a\mid x)=\hat{\pi}^{b}_h(a \mid x)$, where $\hat \pi^\mathrm{b}_h(a \mid x)$ is a kernel estimator based on $\cD^{\mathrm{hst}}$, the asymptotic MSE of $\hat R_{\mathrm{IPWCS}}(\epol)$ is $\Upsilon(\epol)$. 
\end{theorem}

\paragraph{DM Estimator:} Finally, we define a nonparametric DM estimator $\hat R_{\dm}(\epol)$ as $\ts \E_{\nev}[\E_{\epol(a\mid Z)}[\hat f(a,Z)\mid Z]]$. This estimator achieves the efficiency bound. 
\begin{theorem}[Informal]
\label{thm:dm}
When $\hat f_h(a,x)$ is a kernel estimator based on $\cD^{\mathrm{hst}}$, the asymptotic MSE of $\hat R_{\dm}(\epol)$ is $\Upsilon(\epol)$.
\end{theorem}

\begin{table*}[!]
    \centering
     \caption{Comparison of estimators.  
      Parentheses mean that efficiency is ensured when using specific estimators for nuisances such as kernel estimators. Non-Donsker means whether any non-Donsker type complex estimators can be allowed to plug-in with a valid theoretical guarantee. All of the estimators here do not require any parametric model assumptions. }
     \scalebox{0.80}[0.80]{
    \begin{tabular}{c|ccccc } \toprule
        Estimator & Efficiency & Double Robustness  & Nuisance Functions& Without Oracle of $\bpol(x)$ & Non-Donsker  \\
         \hline
        $\hat R_{\mathrm{IPWCSB}}(\epol)$ &        &   & $r$ &  \\
        $\hat R_{\mathrm{IPWCS}}(\epol)$  &    ( \checkmark )      & &  $r,w$ & \checkmark &  \\
    $\hat R_{\dm}(\epol)$  &   ( \checkmark )     &  &  $f$ & \checkmark & \\
    $\hat{R}_{\mathrm{DRCS}}(\epol)$  &  \checkmark  &    \checkmark   &   $r,w,f$ &  \checkmark &\checkmark  \\    \bottomrule
    \end{tabular}
    }
    \label{tab:comparison}
\end{table*}

\subsection{Comparison of Estimators} 
We compare the estimators discussed so far.
This is also summarized in Table \ref{tab:comparison}. First, the estimator $\hat R_{\mathrm{DRCS}}$ allows any non-Donsker type complex estimators with lax convergence rate conditions of the nuisance estimators. However, the analyses of $\hat R_{\mathrm{IPWCS}}$ and $\hat R_{\dm}$ are specific to the kernel estimators though the asymptotic MSE of $\hat R_{\mathrm{IPWCS}},\,\hat R_{\dm}$ and $\hat R_{\mathrm{DRCS}}$ are the same in this special case. When the kernel estimators are replaced with any non-Donsker type complex estimators, the rate condition $\|\hat r(X)\hat w(A,X)-r(X)w(A,X)\|_2=\op(n^{-1/4})$ or $\|\hat f(A,X)-f(A,X)\|_2=\op(n^{-1/4})$ \emph{cannot} guarantee the $\sqrt{n}$-consistency and efficiency even if we use cross-fitting. Second, the only $\hat{R}_{\mathrm{DRCS}}$ has double robustness; however, $\hat R_{\mathrm{IPWCS}}$ and $\hat R_{\dm}$ do not have this property.

\paragraph{Comparison among IPW estimators:}
We can observe that the asymptotic MSE of $\hat R_{ \mathrm{IPWCS}}$\footnote{In this paragraph, we omit $\epol$ from the estimator $\hat R(\epol)$.} is smaller than that of $\hat R_{\mathrm{IPWCSB}}$. This result looks unusual since $\hat R_{\mathrm{IPWCSB}}$ uses more knowledge than $\hat R_{ \mathrm{IPWCS}}$. The intuitive reason for this fact is that $\hat R_{ \mathrm{IPWCS}}$ is considered to be using control variate.  The same paradox is known in other works of causal inference \citep{RobinsJM1992Eeeb}. 
Note that this fact does not imply $\hat R_{ \mathrm{IPWCS}}$ is superior to $\hat R_{\mathrm{IPWCSB}}$ since more smoothness conditions are required in $\hat R_{\mathrm{IPWCS}}$, and this can be violated in practice \citep{RobinsJamesM.1997TACO}.


\section{OPL under a Covariate Shift}
\label{sec:opl}
In this section, we propose OPL estimators based on the doubly robust estimator $\hat R_{\mathrm{DRCS}}(\epol)$ to estimate the optimal policy that maximizes the expected reward over the evaluation data. The optimal policy $\pi^*$ is defined as $
\pi^* = \argmax_{\pi\in \Pi} R(\pi)$, where recall that $\Pi$ is a policy class. By applying each OPE estimator, we can define the following estimators: 
\begin{align*}\ts
    &\hat \pi_{\mathrm{DRCS}} =\argmax_{\pi\in \Pi} \hat{R}_{\mathrm{DRCS}}(\pi),\ \ \ \hat \pi_{\mathrm{DM}} =\argmax_{\pi\in \Pi} \hat{R}_{\mathrm{
    DM}}(\pi),\ \ \ \hat \pi_{\mathrm{IPWCS}} =\argmax_{\pi\in \Pi} \hat{R}_{\mathrm{\mathrm{IPWCS}}}(\pi).
\end{align*}

To obtain a theoretical implication, for simplicity, we assume $\mathcal{A}$ is a finite state space and the policy class $\Pi$ is deterministic. Then, for the $\epsilon$-Hamming covering number $N_{H}(\epsilon,\Pi)$ and its entropy integral $\kappa(\Pi):=\int_{0}^{\infty} \sqrt{\log N_{H}(\epsilon^2,\Pi)}\mathrm{d}\epsilon$ \citep{ZhouZhengyuan2018OMPL}, the regret bound of $\hat \pi_{\dml}$ is obtained. 
 
\begin{theorem}[Regret bound of $ \hat \pi_{\dml} $]
\label{thm:regret}
Assume that for any $0< \epsilon < 1$, there exists $\omega$ such that $N_{H}(\epsilon,\Pi)=\bigO( \exp(1/\epsilon)^{\omega})$,\,$0<\omega < 0.5$. Also suppose that for $k\in \{1,\cdots,\xi\}$, $\|\hat r^{(k)}(X)-r(X)\|_2=\op(n^{-1/4})$, $\|1/\hat{\pi}^{(k)\mathrm{b}}(A,X)-1/{\pi}^{\mathrm{b}}(A,X)\|_2=\op(n^{-1/4})$, and $\|\hat f^{(k)}(A,X)-f(A,X)\|_2=\op(n^{-1/4})$. Then, by defining $\Upsilon_{*} =\sup_{\pi \in \Pi}\Upsilon(\pi)$, there exists an integer $N_{\delta}$ such that with probability at least $1-2\delta$, for all $n\geq N_{\delta}$, $$\ts R(\pi^{*})-R(\hat \pi_{\dml})=\bigO((\kappa(\Pi)+\sqrt{\log(1/\delta)})\sqrt{\frac{\Upsilon_{*}}{n}}).$$
\end{theorem}

In comparison to the standard regret results in \citet{swaminathan15a,KitagawaToru2018WSBT}, we do not assume we know the true behavior policy. Because $\hat{R}_{\mathrm{DRCS}}(\pi)$ has the double robust structure, we can obtain the regret bound under weak nonparametric rate conditions 
without assuming the behavior policy is known. Besides, this theorem shows that the variance term is related to attain the low regret. This is achieved by using the efficient estimator $\hat{R}_{\mathrm{DRCS}}(\pi)$. 

\section{Experiments}
\label{sec:exp}
In this section, we demonstrate the effectiveness of the proposed estimators using data obtained with bandit feedback. Following \citet{dudik2011doubly} and \citet{Chow2018}, we evaluate the proposed estimators using the standard classification datasets from the UCI repository by transforming the classification data into contextual bandit data. From the UCI repository, we use the SatImage, Vehicle, and PenDigits datasets \footnote{\url{https://www.csie.ntu.edu.tw/~cjlin/libsvmtools/datasets/multiclass.html}}. For each dataset, we randomly choose $800$ samples (the results with other sample sizes are reported in Appendix~\ref{appdx:details_exp}). First, we classify data into historical and evaluation data with probability defined as $\ts p(hist=+1|X_i) = \frac{C_{\mathrm{prob}}}{1+\exp(-\tau(X_i)+0.1\varepsilon)}$, where $hist=+1$ denotes that the sample $i$ belongs to the historical data, $C_{\mathrm{prob}}$ is a constant, $\varepsilon$ is a random variable that follows the standard normal distribution, $X_{k,i}$ is the $k$-th element of the vector $X_i$, and $\tau(X_i) = \sum^5_{j=1}X_{j,i}$. By adjusting $C_{\mathrm{prob}}$, we classify $70\%$ samples as the historical data and $30\%$ samples as the evaluation data. Thus, we generate the historical and evaluation data under a covariate shift. Then, we make a deterministic policy $\pi_d$ by training a logistic regression classifier on the historical data. We construct three different behavior policies as mixtures of $\pi^d$ and the uniform random policy $\pi^u$ by changing a mixture parameter $\alpha$, i.e., $\bpol = \alpha\pi^d+(1-\alpha)\pi^u$. The candidates of the mixture parameter $\alpha$ are $\{0.7, 0.4, 0.0\}$ as \citet{Kallus2019IntrinsicallyES}. In Section~\ref{subsec:exp_ope}, we show the experimental results of OPE. In Section~\ref{subsec:exp_opl}, we show the experimental results of OPL. In both sections, the historical $p(x)$ and evaluation distributions $q(x)$ are unknown, and the behavior policy $\bpol$ is also unknown. More details, such as the description of the data and choice of hyperparameters, are in Appendix~\ref{appdx:details_exp}. 

\begin{table}[!htbp]
\caption{OPE results. The alphabets (a),(b),(c) refer to the cases where the behavior policies are (a) $0.7\pi^d+0.3\pi^u$, (b) $0.4\pi^d+0.6\pi^u$, (c) $0.0\pi^d+1.0\pi^u$, respectively. The notation -- means each value is larger than $1.0$. }
 \begin{minipage}[t]{.49\textwidth}
 \begin{center}
\caption*{OPE with SatImage dataset} 
\label{tbl:table1}
\scalebox{0.66}[0.66]{
\begin{tabular}{p{2mm}|p{5mm}p{6mm}|p{4mm}p{4.3mm}|p{5mm}p{6mm}|p{5mm}p{6mm}|p{5mm}p{6mm}}
\hline
\multirow{2}{*}{} &   \multicolumn{2}{c|}{DRCS}  & \multicolumn{2}{c|}{IPWCS} & \multicolumn{2}{c|}{DM} & \multicolumn{2}{c|}{IPWCS-R} &  \multicolumn{2}{c}{DM-R} \\
& MSE & SD & MSE & SD & MSE & SD & MSE & SD & MSE & SD \\
\hline
(a) &  0.107 &  0.032 &  -- &   -- &   \textbf{0.042} &  0.043 &  \textbf{0.045} &  0.049 &  0.073 &  0.023 \\
(b) &  \textbf{0.096} &  0.025 &   -- &   -- &  0.134 &  0.052 &  \textbf{0.093} &  0.069 &  0.177 &  0.033 \\
(c) &  \textbf{0.154} &  0.051 &  -- &   -- &  0.336 &  0.079 &  \textbf{0.022} &  0.026 &  0.372 &  0.050 \\
\hline
\end{tabular}
}
\end{center}
\end{minipage}
\begin{minipage}[t]{.49\textwidth}
  \begin{center}
      \caption*{OPE with Vehicle dataset} 
\label{tbl:table2}
\scalebox{0.66}[0.66]{
\begin{tabular}{p{2mm}|p{5mm}p{6mm}|p{4mm}p{4.3mm}|p{5mm}p{6mm}|p{5mm}p{6mm}|p{5mm}p{6mm}}
\hline
\multirow{2}{*}{}   & \multicolumn{2}{c|}{DRCS}  & \multicolumn{2}{c|}{IPWCS} & \multicolumn{2}{c|}{DM} & \multicolumn{2}{c|}{IPWCS-R} &  \multicolumn{2}{c}{DM-R} \\
 & MSE & SD & MSE & SD & MSE & SD & MSE & SD & MSE & SD \\
\hline
(a)  & \textbf{0.029} &  0.019 &   -- &   --&  \textbf{0.038} &  0.035 &  0.568 &  0.319 &  0.040 &  0.014 \\
(b) &    \textbf{0.019} &  0.024 & -- & -- &  \textbf{0.095} &  0.062 &  0.576 &  0.357 &  0.089 &  0.019 \\
(c) &    \textbf{0.037} &  0.030 &  --&   -- &  0.213 &  0.049 &  0.233 &  0.193 &  \textbf{0.210} &  0.031 \\
\hline
\end{tabular}
}

  \end{center}
  \end{minipage}
\end{table}

 \begin{table}[!htbp]
 \caption{OPL results. The alphabets (a),(b), and (c) refer to the cases where the behavior policies are (a) $0.7\pi^d+0.3\pi^u$, (b) $0.4\pi^d+0.6\pi^u$, (c) $0.0\pi^d+1.0\pi^u$, respectively. }
\label{tbl:table4}

\begin{minipage}[t]{.49\textwidth}
 \begin{center}
\caption*{OPL with SatImage dataset} 
\scalebox{0.68}[0.68]{
\begin{tabular}{l|rr|rr|rr}
\hline
 &   \multicolumn{2}{c|}{DRCS}  & \multicolumn{2}{c|}{IPWCS} & \multicolumn{2}{c}{DM}  \\
 &   RWD & SD & RWD & SD & RWD & SD  \\
\hline
(a)  &  \textbf{0.723} &  0.035 &  0.423 &  0.063 &  0.658 &  0.045 \\
(b)  & \textbf{0.710} &  0.035 &  0.482 &  0.096 &  0.641 &  0.048 \\
(c)  & \textbf{0.652} &  0.046 &  0.460 &  0.131 &  0.465 &  0.070 \\
\hline
\end{tabular}
}
\end{center}
\end{minipage}
\begin{minipage}[t]{.49\textwidth}
  \begin{center}
      \caption*{OPL with Vehicle dataset} 
\scalebox{0.68}[0.68]{
\begin{tabular}{l|rr|rr|rr}
\hline
 &   \multicolumn{2}{c|}{DRCS}  & \multicolumn{2}{c|}{IPWCS} & \multicolumn{2}{c}{DM}  \\
 &   RWD & SD & RWD & SD & RWD & SD  \\
\hline
(a)  &  \textbf{0.496} &  0.017 &  0.310 &  0.030 &  0.411 &  0.040 \\
(b)  &  \textbf{0.510} &  0.029 &  0.290 &  0.051 &  0.393 &  0.052 \\
(c)  &  \textbf{0.480} &  0.044 &  0.280 &  0.041 &  0.313 &  0.065 \\
\hline
\end{tabular}
}

  \end{center}
\end{minipage}
\end{table}

\subsection{Experiments of Off-Policy Evaluation}
\label{subsec:exp_ope}
For the experiments, we use an evaluation policy $\epol$ defined as $0.9\pi^d+0.1\pi^u$. Here, we compare the MSEs of five estimators, DRCS, DM, DM-R, IPWCS, and IPWCS-R. DRCS is the proposed estimator $\hat{R}_{\mathrm{DRCS}}$ where we use kernel Ridge regression for estimating $f(a,x)$ and $w(a, x)$ and use KuLISF \citep{kanamori2012kulsif} for $r(x)$. For this estimator, we use $2$-fold cross-fitting. DM denotes the direct method estimator $\hat R_{\mathrm{DM}}(\epol)$ with $f(a,x)$ estimated by 
Nadaraya-Watson regression defined in Section~\ref{sec:other_candidates}. DM-R is the same estimator, but we use the kernel Ridge regression for $f(a,x)$. IPWCS is the IPW estimator $\hat R_{\mathrm{IPWCS}}(\epol)$, where we use kernel regression defined in Section~\ref{sec:other_candidates} to estimate $r(x)$ and $w(a, x)$. IPWCS-R is the same estimator, but we use KuLISF to estimate $r(x)$. Note that nuisance estimators in DM-R and IPWCS-R do not satisfy the Donsker condition. 

The resulting MSE and the standard deviation (SD) over $20$ replications of each experiment are shown in Tables~\ref{tbl:table1}, where we highlight in bold the best two estimators in each case. DRCS generally outperforms the other estimators. This result shows that the efficiency and double robustness of DRCS translate to satisfactory performance. IPW based estimators have unstable performance. While IPWCS-R shows the best performance in SatImage dataset, it has severely low performance for Vehicle dataset. IPWCS has a poor performance in both datasets. The larger instability of IPWCS-R is mainly due to the nuisance estimators in IPWCS-R do not satisfy the Donsker condition. When the behavior policy is similar to the evaluation policy, the DM estimators (DM and DM-R) also work well.

\begin{remark}
For DRCS and IPWCS, we also conducted experiments when we use the self-normalization following \citet{Swaminathan2015b}. The details and results are shown in Appendix~\ref{appdx:self_dml}. We also investigate the performance with an additional dataset in \cref{appdx:details_exp}. 
\end{remark}

\subsection{Experiments of Off-Policy Learning}
\label{subsec:exp_opl}
In the experiments of OPL, we compare the performances of three estimators for the optimal policy maximizing expected reward over the evaluation data: $\hat \pi_{\mathrm{DRCS}}$ with $f(a,x)$ and $w(a, x)$ estimated by kernel Ridge regression and $r(x)$ estimated by KuLISF (DRCS), $\hat \pi_{\mathrm{DM}}$ with $f(a, x)$ estimated by kernel regression defined in Section~\ref{sec:other_candidates} (DM), and $\hat \pi_{\mathrm{IPWCS}}$ with $r(x)$ and $w(a, x)$ estimated by kernel regression defined in Section~\ref{sec:other_candidates} (IPWCS). For the policy class $\Pi$, we use a model with the Gaussian kernel defined in \cref{appdx:opl_alg}. For DRCS, we use $2$-fold cross-fitting and add a regularization term. 

We conducted $10$ trials for each experiment. The resulting expected reward over the evaluation data (RWD) and the standard deviation (SD) of estimators for OPL are shown in Table~\ref{tbl:table4}, where we highlight in bold the best estimator in each case. For all cases, the estimator $\hat \pi_{\mathrm{DRCS}}$ outperforms the other estimators. We can find that, when an estimator of OPE shows high performance, a corresponding estimator of OPL also shows high performance. The results show that the statistical efficiency of the OPE estimator translates into better regret performance as in Theorem \ref{thm:regret}.

\vspace{-4mm}
\section{Conclusion and Future Direction}
\vspace{-3mm}
We calculated the efficiency bound for OPE under a covariate shift and proposed OPE and OPL methods for the situation. Especially, DRCS has doubly robustness and achieves the efficiency bound under weak nonparametric rate conditions. 

The proposed OPE estimator is the efficient estimator under the simplest setting in a transportability problem \citep{Bareinboim2016}. Complete identification algorithms have been developed in a more complex setting \citep{NIPS2014_5536}; however, statistical efficient estimation methods have not been considered. Our work opens the door to this new direction. How to conduct efficient estimation in such a complex setting is an interesting future work.

\section*{Broader Impact}

Because the policies in sequential decision-making problems are critical in various real-world applications, the OPE methods are employed to evaluate the new policy and reduce the risk of deploying a poor policy. 
We focus on the OPE under a covariate shift a historical and evaluation data, which has many practical applications. For example, in the advertising applications, we usually deliver ads only in the particular region as the test-marketing at the beginning, then expand to other regions that have different feature distribution. Thus we face the covariate shift in the evaluation and training a new policy for the new region. 

Though its practical importance, the OPE methods under the covariate shift have not been researched well, and people apply standard  OPE methods to cases under the covariate shift. As we explained, the standard methods are not robust against the covariate shift. Among the standard methods, the IPW estimator is not consistent, and the DM and DR estimator has consistency when the model of conditional outcome is correct. In particular, under a covariate shift, the standard DR estimator is \emph{not} doubly robust; namely, it is consistent only when the model of conditional outcome is correct. Thus, the standard estimator has a potential risk to mislead the user's decision making and might cause serious problems in the industry because many decision makings such as ad-optimization rely on the result of the evaluation. On the other hand, the proposed estimator has a doubly robust property. This robustness helps to avoid such potential risk of incorrect decision making.

\bibliographystyle{arxiv}
\bibliography{CS-OPE,pfi3}

\begin{thebibliography}{58}
\providecommand{\natexlab}[1]{#1}
\providecommand{\url}[1]{\texttt{#1}}
\expandafter\ifx\csname urlstyle\endcsname\relax
  \providecommand{\doi}[1]{doi: #1}\else
  \providecommand{\doi}{doi: \begingroup \urlstyle{rm}\Url}\fi

\bibitem[Athey \& Wager(2017)Athey and Wager]{AtheySusan2017EPL}
Athey, S. and Wager, S.
\newblock Efficient policy learning.
\newblock \emph{arXiv:1702.02896}, 2017.

\bibitem[Bareinboim \& Pearl(2014)Bareinboim and Pearl]{NIPS2014_5536}
Bareinboim, E. and Pearl, J.
\newblock Transportability from multiple environments with limited experiments:
  Completeness results.
\newblock In \emph{NeurIPS}. 2014.

\bibitem[Bareinboim \& Pearl(2016)Bareinboim and Pearl]{Bareinboim2016}
Bareinboim, E. and Pearl, J.
\newblock Causal inference and the data-fusion problem.
\newblock \emph{Proceedings of the National Academy of Sciences of the United
  States of America}, 2016.

\bibitem[Beygelzimer \& Langford(2009)Beygelzimer and Langford]{kdd2009_ads}
Beygelzimer, A. and Langford, J.
\newblock The offset tree for learning with partial labels.
\newblock In \emph{KDD}, pp.\  129^^e2^^80^^93138, 2009.

\bibitem[Bibaut et~al.(2019)Bibaut, Malenica, Vlassis, and Van
  Der~Laan]{pmlr-v97-bibaut19a}
Bibaut, A., Malenica, I., Vlassis, N., and Van Der~Laan, M.
\newblock More efficient off-policy evaluation through regularized targeted
  learning.
\newblock In \emph{ICML}, volume~97, pp.\  654--663, 2019.

\bibitem[Bickel et~al.(1998)Bickel, Klaassen, Ritov, and Wellner]{bickel98}
Bickel, P.~J., Klaassen, C. A.~J., Ritov, Y., and Wellner, J.~A.
\newblock \emph{Efficient and Adaptive Estimation for Semiparametric Models}.
\newblock Springer, 1998.

\bibitem[Cheng(1994)]{cheng1994}
Cheng, P.~E.
\newblock Nonparametric estimation of mean functionals with data missing at
  random.
\newblock \emph{Journal of the American Statistical Association}, 89\penalty0
  (425):\penalty0 81--87, 1994.

\bibitem[Chernozhukov et~al.(2018)Chernozhukov, Chetverikov, Demirer, Duflo,
  Hansen, Newey, and Robins]{ChernozhukovVictor2018Dmlf}
Chernozhukov, V., Chetverikov, D., Demirer, M., Duflo, E., Hansen, C., Newey,
  W., and Robins, J.
\newblock Double/debiased machine learning for treatment and structural
  parameters.
\newblock \emph{Econometrics Journal}, 21:\penalty0 C1--C68, 2018.

\bibitem[Chernozhukov et~al.(2019)Chernozhukov, Demirer, Lewis, and
  Syrgkanis]{Chernozhukov2019}
Chernozhukov, V., Demirer, M., Lewis, G., and Syrgkanis, V.
\newblock Semi-parametric efficient policy learning with continuous actions.
\newblock In \emph{NeurIPS}. 2019.

\bibitem[Cole \& Stuart(2010)Cole and Stuart]{ColeStephenR.2010GEFR}
Cole, S.~R. and Stuart, E.~A.
\newblock Generalizing evidence from randomized clinical trials to target
  populations.
\newblock \emph{American Journal of Epidemiology}, 172\penalty0 (1):\penalty0
  107--115, 2010.

\bibitem[Dahabreh et~al.(2019)Dahabreh, Robertson, Tchetgen, Stuart, and
  Hern^^c3^^a1n]{DahabrehIssaJ.2019Gcif}
Dahabreh, I.~J., Robertson, S.~E., Tchetgen, E.~J., Stuart, E.~A., and
  Hern^^c3^^a1n, M.~A.
\newblock Generalizing causal inferences from individuals in randomized trials
  to all trial‐eligible individuals.
\newblock \emph{Biometrics}, 2019.

\bibitem[D^^c3^^adaz(2019)]{DiazIvan2019Mlit}
D^^c3^^adaz, I.
\newblock Machine learning in the estimation of causal effects: targeted
  minimum loss-based estimation and double/debiased machine learning.
\newblock \emph{Biostatistics}, 2019.

\bibitem[Dud{\'\i}k et~al.(2011)Dud{\'\i}k, Langford, and Li]{dudik2011doubly}
Dud{\'\i}k, M., Langford, J., and Li, L.
\newblock Doubly {R}obust {P}olicy {E}valuation and {L}earning.
\newblock In \emph{Proceedings of the 28th International Conference on Machine
  Learning}, pp.\  1097--1104, 2011.

\bibitem[Farajtabar et~al.(2018)Farajtabar, Chow, and Ghavamzadeh]{Chow2018}
Farajtabar, M., Chow, Y., and Ghavamzadeh, M.
\newblock More robust doubly robust off-policy evaluation.
\newblock \emph{ICML}, 2018.

\bibitem[Hahn(1998)]{HahnJinyong1998OtRo}
Hahn, J.
\newblock On the role of the propensity score in efficient semiparametric
  estimation of average treatment effects.
\newblock \emph{Econometrica}, 66:\penalty0 315--331, 1998.

\bibitem[Hirano et~al.(2003)Hirano, Imbens, and Ridder]{hirano2003efficient}
Hirano, K., Imbens, G.~W., and Ridder, G.
\newblock Efficient estimation of average treatment effects using the estimated
  propensity score.
\newblock \emph{Econometrica}, 71\penalty0 (4):\penalty0 1161--1189, 2003.

\bibitem[Horvitz \& Thompson(1952)Horvitz and Thompson]{Horvitz1952}
Horvitz, D.~G. and Thompson, D.~J.
\newblock A generalization of sampling without replacement from a finite
  universe.
\newblock \emph{Journal of the American Statistical Association}, 1952.

\bibitem[Johansson et~al.(2016)Johansson, Shalit, and
  Sontag]{pmlr-v48-johansson16}
Johansson, F., Shalit, U., and Sontag, D.
\newblock Learning representations for counterfactual inference.
\newblock In \emph{ICML}, 2016.

\bibitem[Johansson et~al.(2018)Johansson, Kallus, Shalit, and
  Sontag]{JohanssonFredrik2018LWRf}
Johansson, F., Kallus, N., Shalit, U., and Sontag, D.
\newblock Learning weighted representations for generalization across designs.
\newblock \emph{arXiv:1802.08598}, 2018.

\bibitem[Kallus \& Uehara(2019)Kallus and Uehara]{Kallus2019IntrinsicallyES}
Kallus, N. and Uehara, M.
\newblock Intrinsically efficient, stable, and bounded off-policy evaluation
  for reinforcement learning.
\newblock In \emph{NeurIPS}. 2019.

\bibitem[Kallus \& Uehara(2020)Kallus and Uehara]{KallusUehara2019}
Kallus, N. and Uehara, M.
\newblock Double reinforcement learning for efficient off-policy evaluation in
  markov decision processes.
\newblock \emph{Journal of Machine Learning Research}, 2020.

\bibitem[Kanamori et~al.(2012)Kanamori, Suzuki, and
  Sugiyama]{kanamori2012kulsif}
Kanamori, T., Suzuki, T., and Sugiyama, M.
\newblock Statistical analysis of kernel-based least-squares density-ratio
  estimation.
\newblock \emph{Mach. Learn.}, 86\penalty0 (3):\penalty0 335^^e2^^80^^93367,
  2012.

\bibitem[Kennedy(2019)]{KennedyEdwardH2019NCEB}
Kennedy, E.~H.
\newblock Nonparametric causal effects based on incremental propensity score
  interventions.
\newblock \emph{Journal of the American Statistical Association}, 2019.

\bibitem[Kitagawa \& Tetenov(2018)Kitagawa and Tetenov]{KitagawaToru2018WSBT}
Kitagawa, T. and Tetenov, A.
\newblock Who should be treated? empirical welfare maximization methods for
  treatment choice.
\newblock \emph{Econometrica}, 86:\penalty0 591--616, 2018.

\bibitem[Klaassen(1987)]{klaassen1987}
Klaassen, C. A.~J.
\newblock Consistent estimation of the influence function of locally
  asymptotically linear estimators.
\newblock \emph{Annals of Statistics}, 15, 1987.

\bibitem[Li et~al.(2010)Li, Chu, Langford, and Schapire]{www2010_cb}
Li, L., Chu, W., Langford, J., and Schapire, R.~E.
\newblock A contextual-bandit approach to personalized news article
  recommendation.
\newblock In \emph{WWW}, 2010.

\bibitem[Mu^^c3^^b1oz \& Van Der~Laan(2012)Mu^^c3^^b1oz and Van
  Der~Laan]{MunozIvanDiaz2012PICE}
Mu^^c3^^b1oz, I.~D. and Van Der~Laan, M.
\newblock Population intervention causal effects based on stochastic
  interventions.
\newblock \emph{Biometrics}, 2012.

\bibitem[Nadaraya(1964)]{Nada:1964}
Nadaraya, E.~A.
\newblock On estimating regression.
\newblock \emph{Theory of Probability and its Applications}, 9:\penalty0
  141--142, 1964.

\bibitem[Narita et~al.(2019)Narita, Yasui, and Yata]{narita2019counterfactual}
Narita, Y., Yasui, S., and Yata, K.
\newblock Efficient counterfactual learning from bandit feedback.
\newblock \emph{AAAI}, 2019.

\bibitem[Newey \& Mcfadden(1994)Newey and Mcfadden]{newey94}
Newey, W.~K. and Mcfadden, D.~L.
\newblock Large sample estimation and hypothesis testing.
\newblock \emph{Handbook of Econometrics}, IV:\penalty0 2113--2245, 1994.

\bibitem[Oberst \& Sontag(2019)Oberst and Sontag]{Oberst2019}
Oberst, M. and Sontag, D.
\newblock Counterfactual off-policy evaluation with gumbel-max structural
  causal models.
\newblock In \emph{ICML}, 2019.

\bibitem[Pearl \& Bareinboim(2011)Pearl and Bareinboim]{PearlJ2011ToCa}
Pearl, J. and Bareinboim, E.
\newblock Transportability of causal and statistical relations: A formal
  approach.
\newblock In \emph{ICDM Workshops}, 2011.

\bibitem[Pearl \& Bareinboim(2014)Pearl and Bareinboim]{PearlJudea2015EVFD}
Pearl, J. and Bareinboim, E.
\newblock External validity: From do-calculus to transportability across
  populations.
\newblock \emph{Statistical Science}, 29, 2014.

\bibitem[Qin(1998)]{QinJing1998IfCa}
Qin, J.
\newblock Inferences for case-control and semiparametric two-sample density
  ratio models.
\newblock \emph{Biometrika}, 1998.

\bibitem[Reddi et~al.(2015)Reddi, Poczos, and Smola]{Reddi2015}
Reddi, S.~J., Poczos, B., and Smola, A.
\newblock Doubly robust covariate shift correction.
\newblock In \emph{AAAI}, 2015.

\bibitem[Robins \& Ritov(1997)Robins and Ritov]{RobinsJamesM.1997TACO}
Robins, J.~M. and Ritov, Y.~A.
\newblock Toward a curse of dimensionality appropriate (coda) asymptotic theory
  for semi‐parametric models.
\newblock \emph{Statistics in Medicine}, 1997.

\bibitem[Robins et~al.(1992)Robins, Mark, and Newey]{RobinsJM1992Eeeb}
Robins, J.~M., Mark, S.~D., and Newey, W.~K.
\newblock Estimating exposure effects by modelling the expectation of exposure
  conditional on confounders.
\newblock \emph{Biometrics}, 1992.

\bibitem[Robins et~al.(1994)Robins, Rotnitzky, and Zhao]{robins94}
Robins, J.~M., Rotnitzky, A., and Zhao, L.~P.
\newblock Estimation of regression coefficients when some regressors are not
  always observed.
\newblock \emph{Journal of the American Statistical Association}, 89:\penalty0
  846--866, 1994.

\bibitem[Rubin(1987)]{rubin87}
Rubin, D.~B.
\newblock \emph{Multiple Imputation for Nonresponse in Surveys}.
\newblock Wiley, New York, 1987.

\bibitem[Shimodaira(2000)]{shimodaira2000improving}
Shimodaira, H.
\newblock Improving predictive inference under covariate shift by weighting the
  log-likelihood function.
\newblock \emph{Journal of statistical planning and inference}, 90\penalty0
  (2):\penalty0 227--244, 2000.

\bibitem[Sondhi et~al.(2020)Sondhi, Arbour, and Dimmery]{pmlr-v108-sondhi20a}
Sondhi, A., Arbour, D., and Dimmery, D.
\newblock Balanced off-policy evaluation in general action spaces.
\newblock In \emph{AISTATS}, 2020.

\bibitem[Sugiyama et~al.(2008)Sugiyama, Nakajima, Kashima, Buenau, and
  Kawanabe]{NIPS2007_3248}
Sugiyama, M., Nakajima, S., Kashima, H., Buenau, P.~V., and Kawanabe, M.
\newblock Direct importance estimation with model selection and its application
  to covariate shift adaptation.
\newblock In \emph{NeurIPS}. 2008.

\bibitem[Sugiyama et~al.(2012)Sugiyama, Suzuki, and
  Kanamori]{Sugiyama:2012:DRE:2181148}
Sugiyama, M., Suzuki, T., and Kanamori, T.
\newblock \emph{Density Ratio Estimation in Machine Learning}.
\newblock 2012.

\bibitem[Swaminathan \& Joachims(2015{\natexlab{a}})Swaminathan and
  Joachims]{Swaminathan2015b}
Swaminathan, A. and Joachims, T.
\newblock The self-normalized estimator for counterfactual learning.
\newblock In \emph{NeurIPS}. 2015{\natexlab{a}}.

\bibitem[Swaminathan \& Joachims(2015{\natexlab{b}})Swaminathan and
  Joachims]{swaminathan15a}
Swaminathan, A. and Joachims, T.
\newblock Batch learning from logged bandit feedback through counterfactual
  risk minimization.
\newblock \emph{Journal of Machine Learning Research}, 2015{\natexlab{b}}.

\bibitem[Tripathi(1999)]{matrix}
Tripathi, G.
\newblock {A matrix extension of the Cauchy-Schwarz inequality}.
\newblock \emph{Economics Letters}, 63:\penalty0 1--3, 1999.

\bibitem[Tsiatis(2006)]{TsiatisAnastasiosA2006STaM}
Tsiatis, A.~A.
\newblock \emph{Semiparametric Theory and Missing Data}.
\newblock Springer Series in Statistics. Springer New York, New York, NY, 2006.

\bibitem[van~der Vaart(1998)]{VaartA.W.vander1998As}
van~der Vaart, A.~W.
\newblock \emph{Asymptotic statistics}.
\newblock Cambridge University Press, Cambridge, UK, 1998.

\bibitem[Wainwright(2019)]{WainwrightMartinJ2019HS:A}
Wainwright, M.~J.
\newblock \emph{High-Dimensional Statistics : A Non-Asymptotic Viewpoint}.
\newblock Cambridge University Press, New York, 2019.

\bibitem[Wang et~al.(2017)Wang, Agarwal, and Dudik]{wang2017optimal}
Wang, Y.-X., Agarwal, A., and Dudik, M.
\newblock Optimal and adaptive off-policy evaluation in contextual bandits.
\newblock 2017.

\bibitem[Watson(1964)]{Wats:1964}
Watson, G.~S.
\newblock Smooth regression analysis.
\newblock \emph{Sankhy\=a Ser.}, 26:\penalty0 359--372, 1964.

\bibitem[Wooldridge(2001)]{WooldridgeJeffreyM.2001APOW}
Wooldridge, J.~M.
\newblock Asymptotic properties of weighted m -estimators for standard
  stratified samples.
\newblock \emph{Econometric Theory}, 2001.

\bibitem[Young et~al.(2014)Young, Her^^c5^^84an, and
  Robins]{YoungJessicaG2014Ieaa}
Young, J.~G., Her^^c5^^84an, M.~A., and Robins, J.~M.
\newblock Identification, estimation and approximation of risk under
  interventions that depend on the natural value of treatment using
  observational data.
\newblock \emph{Epidemiologic methods}, 2014.

\bibitem[Zhang et~al.(2013{\natexlab{a}})Zhang, Tsiatis, Laber, and
  Davidian]{ZhangBaqun2013Reoo}
Zhang, B., Tsiatis, A.~A., Laber, E.~B., and Davidian, M.
\newblock Robust estimation of optimal dynamic treatment regimes for sequential
  treatment decisions.
\newblock \emph{Biometrika}, 2013{\natexlab{a}}.

\bibitem[Zhang et~al.(2013{\natexlab{b}})Zhang, Sch^^c3^^b6lkopf, Muandet, and
  Wang]{pmlr-v28-zhang13d}
Zhang, K., Sch^^c3^^b6lkopf, B., Muandet, K., and Wang, Z.
\newblock Domain adaptation under target and conditional shift.
\newblock In \emph{ICML}, 2013{\natexlab{b}}.

\bibitem[Zhao et~al.(2012)Zhao, Zeng, Rush, and Kosorok]{ZhaoYingqi2012EITR}
Zhao, Y., Zeng, D., Rush, A.~J., and Kosorok, M.~R.
\newblock Estimating individualized treatment rules using outcome weighted
  learning.
\newblock \emph{Journal of the American Statistical Association}, 2012.

\bibitem[Zheng \& van~der Laan(2011)Zheng and van~der
  Laan]{ZhengWenjing2011CTME}
Zheng, W. and van~der Laan, M.~J.
\newblock Cross-validated targeted minimum-loss-based estimation.
\newblock In \emph{Targeted Learning: Causal Inference for Observational and
  Experimental Data}. 2011.

\bibitem[Zhou et~al.(2018)Zhou, Athey, and Wager]{ZhouZhengyuan2018OMPL}
Zhou, Z., Athey, S., and Wager, S.
\newblock Offline multi-action policy learning: Generalization and
  optimization.
\newblock \emph{arxiv:1810.04778}, 2018.

\end{thebibliography}

\newpage
\appendix

\section{Notations, Terms, and Abbreviations}\label{sec:notation}
In this section, we summarize the notations used in this paper. 
\begin{table}[h]
    \centering
    \caption{Summary of notations}
    \label{tbl:sum_not}
    \begin{tabular}{l|l}
    $A,\,X,\,Y$ & Action, covariate, reward \\ 
    $\E[\mu(X,A,Y)]$ & $\E_{p(x)\bpol(a \mid x)p(y\mid a,x)}[\mu(x,a,y)]$\\
     $\E[\mu(Z)]$ & $\E_{q(z)}[\mu(z)]$ \\
     $\bpol(a\mid x)$ &  Behavior policy \\
     $\epol(a\mid x)$ &  Evaluation policy \\
     $R(\epol)$ & $\E_{q(x)\epol(a\mid x)p(y\mid a,x)}[y]$ \\
    $r(x)$ & $p(x)/q(x)$ \\ 
     $w(a,x)$ & $\epol(a\mid x)\bpol(a\mid x)$ \\ 
     $f(a,x)$ & $\E_{p(y\mid a,x)}[y \mid a,x]$ \\
     $v(x)$ & $\E_{\epol(a\mid x)}[f(a,x)\mid x]$ \\
     $\Asmse[\hat R]$ &  $\lim_{n\to \infty} \E[(\hat R-R)^2]n $\\
     $\Pi$ & Policy class  \\ 
     $\otimes A$ & $AA^{\top}$ \\
     $\|\mu(X)\|_2,\,\|\mu(X)\|_{\infty}$ & $L^{2}$-norm,\, $L^{\infty}$-norm \\
     $k(\Pi)$ & Entropy integral of $\Pi$ w.r.t $\epsilon$-Hamming distance \\
     $\nhs$ & Number of training data\\  
     $\nev$ & Number of evaluation data \\
     $ A\lessapprox B$ & There exists an absolute constant $C$ s.t. $A\leq CB$ \\ 
     $C_1,C_2,R_{\max}$ & Upper bound of $r(X),w(A,X),Y$ \\ 
     $\rho$ &  $\nhs/(\nhs+\nev)$\\
     $\cD^{\mathrm{hst}}$, $\cD^{\mathrm{eval}}$ & Train data, evaluation data \\
     $\nhs_1$,$\nhs_2$ & Split train data \\
    $\nev_1$,$\nev_2$ & Split evaluation data \\
    $\cD_i$ & Concatenation of $\nhs_i$ and $\nev_i$\,$i=1,2$ \\
    $\hG_{\nhs}$ &  $\sqrt{\nhs}\{\E_{\nhs}-\E\}$ Empirical process based on train data \\ 
    $\hG_{\nev}$ &  $\sqrt{\nev}\{\E_{\nev}-\E\}$ Empirical process based on evaluation data \\ 
    $\Upsilon(\epol)$ & Semiaprametric lower bound of $R(\thpol)$ under nonparametric model\\ 
    $K_h(\cdot) $ & Kernel with a bandwidth $h$ \\ 
    $\nhs_k$ & $k$-th train data \\
    $\nev_k$ &  $k$-th evaluation data\\
    \end{tabular}
    \label{tab:notation}
\end{table}

\section{Identification under Potential Outcome Framework}
\label{sec:idenfication}
We explain how to apply our results in the main draft under potential outcome framework, which is a common framework in causal inference literature \citep{rubin87}. In this section, our goal is justifying DM and IPWCS estimators under potential outcome framework. 

Let us denote counterfactual variables based on stochastic policies (interventions) as $Y(B)$, where $B$ is a random variable from the conditional density $\epol(b|Z)$ \footnote{The reason why we use $B$ is to distinguish it from the random variable $A$.} and $Z$ is a random variable following the evaluation density $q(z)$. Here, note what we can observe is data:
\begin{align*}
\{X_i,A_i,Y_i\}_{i=1}^{\nhs}\sim p(x)\epol(a\mid x)p(y\mid a,x),\ \{Z_j\}^{\nev}_{j=1}\sim q(z).
\end{align*}
A detailed review of the stochastic intervention is shown in \citet{MunozIvanDiaz2012PICE,YoungJessicaG2014Ieaa}. 

Then, let us put the following assumptions: 
\begin{itemize}
    \item Consistency: $Y=Y(a)$ if $A=a$ for $\forall a \in \mathcal{A}$, 
    \item Unconfoundedness: $A$ and $Y(a)$ are conditionally independent given $X$ for any $a\in \mathcal{A}$,\,$G$ and $Y(a)$ are conditionally independent given $Z$ for any $a\in \mathcal{A}$,
    \item Transportability: $\rE[Y(a)\mid Z=c]=\rE[Y(a)\mid X=c]$ for any $a\in \mathcal{A},c\in \mathcal{X}$.
\end{itemize}
Note that transportability is a weaker assumption compared with the assumption in the main draft:
\begin{align*}
    p_{\mathrm{train}}(Y(a)\mid c)=p_{\mathrm{test}}(Y(a)\mid c),  
\end{align*}
where $p_{\mathrm{train}}(\cdot \mid \cdot)$ is a condition density of $Y(a)$ given $Z$,  $p_{\mathrm{test}}(\cdot \mid \cdot)$ is a condition density of $Y(a)$ given $X$. 
Following Lemma~1 \citep{KennedyEdwardH2019NCEB}, we can prove the following lemma. 
\begin{lemma}[G-formula]
\label{lem:g-formula}
$\E[Y(B)]=\int \E[Y\mid A=a,X=x]\epol(a\mid x)q(x)\mathrm{d}(a,x)$. 
\end{lemma}
\begin{proof}
\begin{align*}\ts
    \E[Y(B)] &=\int  \E[Y(b)\mid B=b,Z=z]\epol(b\mid z)q(z)\mathrm{d}(b,z) \\
    &=\int  \E[Y(g)\mid Z=z]\epol(b\mid z)q(z)\mathrm{d}(b,z) \\
    &=\int  \E[Y(g)\mid X=z]\epol(b\mid z)q(z)\mathrm{d}(b,z) \\
    &=\int  \E[Y(g)\mid A=g,X=z]\epol(b\mid z)q(z)\mathrm{d}(b,z) \\
    &=\int  \E[Y(a)\mid A=a,X=x]\epol(a\mid x)q(x)\mathrm{d}(a,x) \\
    &=\int  \E[Y\mid A=a,X=x]\epol(a\mid x)q(x)\mathrm{d}(a,x). 
\end{align*}
From the first line to the second line, we use a uncounfedness assumption. From the second line to the third line, we use a transportability assumption. From the third line to the fourth line, we use a uncounfedness assumption. From the fourth line to the fifth line, the random variables $a,x$ are replaced with $b,z$. From the fifth line to the sixth line, we use a consistency assumption. 

\end{proof}

From this lemma, the DM method can be naturally introduced. Note this is equivalent to a transport formula \citet[(3.1)]{PearlJudea2015EVFD} when the evaluation policy is atomic. The G-formula described here is its extension when the evaluation policy is stochastic.

\begin{theorem}[IPWCS]
$\E[Y(B)]=\E[r(X)w(A,X)Y]$
\end{theorem}
\begin{proof}
\begin{align*}\ts
    \E[r(X)w(A,X)Y] &=\E[r(X)w(A,X)\E[Y\mid A,X]]  \\ 
    &= \int \E[Y\mid A=a,X=x]r(x)w(a,x)\bpol(a\mid x)p(x)\mathrm{d}(a,x) \\ 
    &= \int \E[Y\mid A=a,X=x]\epol(a\mid x)q(x)\mathrm{d}(a,x) \\
    &= \E[Y(B)]. 
\end{align*}
From the third line to the fourth line, we use a Lemma \ref{lem:g-formula}. 
\end{proof}

\section{Density Ratio Estimation}\label{appdx:uLSIF}
Here, we introduce the formulation of LSIF. In LSIF, we estimate the density ratio $r(x)=\frac{q(x)}{p(x)}$ directly. Let $\mathcal{S}$ be the class of non-negative measurable functions $s:\mathcal{X}\to \mathbb{R}^+$. We consider minimizing the following squared error between $s$ and $r$:
\begin{align}\ts
\label{dr}
&\mathbb{E}_{p(x)}[(s(x) - r(x))^{2}] = \mathbb{E}_{p(x)}[(r(x))^{2}] - 2\mathbb{E}_{q(z)}[s(z)] + \mathbb{E}_{p(x)}[(s(x))^{2}]. 
\end{align}
The first term of the last equation does not affect the result of minimization and we can ignore the term, i.e., the density ratio is estimated through the following minimization problem: 
\begin{align*}\ts
s^{*}& = \argmin_{s\in\mathcal{S}} \left[\frac{1}{2} \mathbb{E}_{p(x)}[(s(x))^{2}] - \mathbb{E}_{q(z)}[s(z)]\right],
\end{align*}
where $\mathcal{S}$ is a hypothesis class of the density ratio. As mentioned above, to minimize the empirical version of (\ref{dr}), we use uLSIF \citep{Sugiyama:2012:DRE:2181148}. Given a hypothesis class $\mathcal{H}$, we obtain $\hat{r}$ by $\hat{r} = \argmin_{s\in\mathcal{H}} \Bigg[\frac{1}{2} {\mathbb{E}}_{\nhs}[(s(X))^{2}]- {\mathbb{E}}_{\nev}[s(Z)] + \mathcal{R}(s)\Bigg]$, where $\mathcal{R}$ is a regularization term. For a model of uLSIF, \citet{kanamori2012kulsif} proposed using kernel based hypothesis to estimate the density ratio nonparametrically. \citet{kanamori2012kulsif} called uLSIF with kernel based hypothesis as KuLSIF. \citet{kanamori2012kulsif} showed that, under some assumptions, the convergence rate of KuLSIF is $\left\|\hat{r}(X) - \left(\frac{q\left(X\right)}{p\left(X\right)}\right)\right\|_2 = \mathrm{O}_p\left(\min\left(n^{\mathrm{hst}}, n^{\mathrm{evl}}\right)^{-\frac{1}{2+\gamma}}\right)$, where $0 < \gamma < 2$ is a constant depending on the bracketing entropy of $\mathcal{H}$.  

\section{Efficiency bound for the stratified sampling mechanism}
\label{sec:semi_ld}

In this section, we discuss the efficiency bound. 

\subsection{\Cramer-Rao lower bound}

First, we show the \Cramer-Rao lower bound when the DGP is a stratified sampling with the historical data $\{\alpha_i\}_{i=1}^{\nhs}$ and evaluation data $\{\beta_i\}_{i=1}^{\nev}$, where $\alpha_i$ and $\beta_i$ are random variables. Let $H_{\nhs}$ and $G_{\nev}$ be the distributions of $\{\alpha_i\}_{i=1}^{\nhs}$ and $\{\beta_i\}_{i=1}^{\nev}$. Let us define a set of densities as $\cm_n=\{H_{\nhs},G_{\nev}\}$. A model $\cm_n^{\para}$ is called a regular parametric submodel if the model can be written as $\cm_n^{\para}=\{H_{\theta_1,\nhs},G_{\theta_2,\nev}\}$, where $\theta_1 \in \Theta_1,\,\theta_2 \in \Theta_2$ and it matches the true distribution at $\theta^{*}_1$ and $\theta^{*}_2$, and it has a density $$\ts h_{\theta_1,\nhs}(\{\alpha_i\})=\prod_{i=1}^{\nhs} h(\alpha_i;\theta_1),\,g_{\theta_2,{\nev}}(\{\beta_i\}))=\prod_{i=1}^{\nev} g(\beta_i;\theta_2).$$ Let $R(H,G)\to \Rl$ be a target functional. Then, the \Cramer-Rao lower bound of the functional $R$ under the parametric submodel $\cm_n^{\para}$ is
\begin{align*}\ts
    \CR(\cm_n^{\para},R)=&\nabla_{\theta_1^{\top}}R( H_{\theta_1},G_{\theta_2})\E[\otimes \nabla_{\theta_1}\log h_{\theta_1,\nhs}]^{-1}\nabla_{\theta_1}R( H_{\theta_1},G_{\theta_2})\\
    &+\nabla_{\theta_2^{\top}}R( H_{\theta_1},G_{\theta_2})\E[\otimes \nabla_{\theta_1}\log g_{\theta_2,\nev}]^{-1}\nabla_{\theta_2}R( H_{\theta_1},G_{\theta_2}).
\end{align*}
Before that, we calculate the \Cramer-Rao lower bound in a tabular setting, where the state, action and reward spaces are finite.  
\begin{theorem}\label{thm:tabular}
In a tabular case, $n\CR(\cm_n^{\para},R)$ is $$\rho^{-1}\E[r^2(X)w^2(A,X)\var[Y\mid A,X]]+(1-\rho)^{-1}\var[v(Z)]$$.
\end{theorem}

\begin{proof}[Proof of Theorem~\ref{thm:tabular}]

In our setting, we have $\{X_i,A_i,Y_i\}_{i=1}^{\nhs}$ and $\{Z_j\}_{j=1}^{\nev}$. The target functional, i.e., the value of the evaluation policy $\epol$ defined in (\ref{def:policy_value}) is
\begin{align}\ts
\label{eq:int_functional}
    R(\epol)=\int yq(x)\epol(a\mid x)p(y\mid a,x)\mathrm{d}\mu(a,x,y), 
\end{align}
where $\mu$ is a baseline measure such as Lebesgue or counting measure. The scaled \Cramer-Rao lower bound for regular parametric models under $\cm^{\para}_n$:
\begin{align*}\ts
    n\mathrm{CR}(\cm^{\mathrm{para}}_n,R)
\end{align*}
is given by
\begin{align}\ts
    &\rho^{-1}A_1^{-1}B^{-1}_1A^{\top}_1+(1-\rho)^{-1}A_2B_2^{-1}A^{\top}_2\, \label{eq:cracra}\\
   A_1 &= \E_{x\sim q(x),\,a\sim \epol(a\mid x),y\sim p(y\mid a,x)}[y\nabla_{\theta^{\top}_1} \log p(y\mid a,x;\theta_1 ) ], \nonumber\\
   A_2 &= \E_{x\sim q(x),\,a\sim \epol(a\mid x),y\sim p(y\mid a,x)}[y\nabla_{\theta^{\top}_2} \log q(x;\theta_2)],  \nonumber\\
   B_1 &= \E_{x\sim p(x),a\sim \bpol(a\mid x),y\sim p(y\mid a,x)}[\otimes \nabla_{\theta_1}  \log p(y\mid a,x;\theta_1 )], \nonumber\\ 
   B_2 &= \E_{z\sim q(z)}[\otimes \nabla_{\theta_2} \log q(z;\theta_2)].  \nonumber
\end{align}
Then, from the Cauchy Schwartz inequality \citep{matrix}, we have the following inequality:
\begin{align*}\ts
   \E[A(Z)B^{\top}(Z)]\E[B(Z)B^{\top}(Z)]^{-1}\E[A(Z)B^{\top}(Z)]^{\top}\leq \E[A^2(Z)],
\end{align*}
where $\E[A(Z)]=0,\E[B(Z)]=0$. Then, we obtain the following upper bound:
\begin{align*}\ts
    &A_1^{-1}B^{-1}_1A^{\top}_1\\
    &=\E[r(X)w(A,X)Y\nabla_{\theta^{\top}_1}\log p(Y\mid A,X;\theta_1)] \E[\otimes \nabla_{\theta_1}\log p(Y\mid A,X;\theta_1) ]^{-1} \\
    &\ \ \ \ \times\E[r(X)w(A,X)Y\nabla_{\theta_1}\log p(Y\mid A,X;\theta_1)]   \\
     &=\E[r(x)w(A,X)\{Y-\E[Y\mid A,X]\}\nabla_{\theta^{\top}_1}\log p(Y\mid A,X;\theta_1)] \E[\otimes \nabla_{\theta_1}\log p(Y\mid A,X;\theta_1) ]^{-1} \\ 
     &\ \ \ \ \times \E[r(x)w(A,X)\{Y-\E[Y\mid A,X]\}\nabla_{\theta_1}\log p(Y\mid A,X;\theta_1)]   \\
    &\leq \E[r^2(X)w^2(A,X)\{Y-\E[Y\mid A,X]\}^2 ]=\E[r^2(X)w^2(A,X)\var[Y\mid A,X] ]. 
\end{align*}
In the same way,
\begin{align*}\ts
     &A_2^{-1}B^{-1}_1A^{\top}_2  \\
    &= \E[v(Z)\nabla_{\theta^{\top}_2} \log q(Z;\theta_2) ]\E[\otimes \nabla_{\theta_2} \log q(Z;\theta_2)]^{-1}\E[v(Z)\nabla_{\theta_2} \log q(z;\theta_2) ] \\
    &=  \E[\{v(Z)-\E[v(Z)]\}\nabla_{\theta^{\top}_2} \log q(Z;\theta_2) ]\E[\otimes \nabla_{\theta_2} \log q(Z;\theta_2)]^{-1}\E[\{v(Z)-\E[v(Z)]\}\nabla_{\theta_2} \log g(Z;\theta_2) ] \\
    &\leq \E[\{v(Z)-\E[v(Z)]\}^2]=\var[v(Z)]. 
\end{align*}
Therefore, 
\begin{align*}\ts
    & \rho^{-1}A_1^{-1}B^{-1}_1A^{\top}_1+(1-\rho)^{-1}A_2B_2^{-1}A^{\top}_2  \\
     &\leq  \rho^{-1}\E[r^2(X)w^2(A,X)\var[Y\mid A,X]]+(1-\rho)^{-1}\var[v(Z)].
\end{align*}
Finally, we have to show this inequality is equality. This is obvious since our setting is tabular. 
\end{proof}

\subsection{Reduction to i.i.d setting}

This \Cramer-Rao lower bound can be extended when the model is semiparametric. Following \citet{TsiatisAnastasiosA2006STaM}, the efficiency bound of a target functional $R$ under the semiparametric model $\cm_n$ is
\begin{align*}\ts
    \lim_{n \to \infty}\sup_{\cm_n^{\para}\subset \cm_n }n \CR(\cm_n^{\para},R). 
\end{align*}
However, since our DGP is not i.i.d, we cannot direct apply standard semiparametric theory here. To circumvent this problem, we regard the whole $n$ data at hand as one sample and consider the case where we observe $m$ samples. Then, as $m$ goes to infinity, the total data size $n':=nm$ goes to infinity. Since each one sample ($n$ data) is i.i.d, we can apply standard semiparametric theory.

We explain the definition of the efficient influence function (EIF). This is a function for one sample $$o=(x_1,\cdots,x_{\nhs},a_1,\cdots,a_{\nhs},y_1,\cdots,y_{\nhs},z_1,\cdots,z_{\nev}).$$ 
This is defined given the target functional and the model. In our context, the EIF has the following property. 
\begin{theorem}{\citep[Chapter 25.20]{VaartA.W.vander1998As}}
The EIF $\phi(o)$ is the gradient of $R(\epol)$  w.r.t the model $\cm_n$, which has the smallest $l_2$-norm. It satisfies that for any regular estimator $\hat R$ of $R(\epol)$ w.r.t the model $\cm_n$, $\mathrm{AMSE}[\hat R]\geq \var[\phi(o)]$, where $\mathrm{AMSE}[\hat R]$ is the second moment of the limiting distribution of $\sqrt{m}(\hat R-R(\epol))$
\end{theorem}

This states that $n\var[\phi(o)]$ is the lower bound in estimating $R(\epol)$. We call $n\var[\phi(o)]$ the efficiency bound. Note that $n$ is fixed here. We consider the asymptotics where $m$ goes to infinity. For the current case, the EIF and efficiency bound are explicitly calculated as follows. 

\begin{theorem}
The EIF of $R(\epol)$ w.r.t the model $\cm_n$ is 
\begin{align*}
\phi(o)=\frac{1}{\nhs}\sum_{i=1}^{\nhs}r(x_i)w(a_i,x_i)\{y_i-q(x_i,a_i)\}+
 \frac{1}{\nev}\sum_{j=1}^{\nev}v(z_j)-R(\pi)    
\end{align*}
The (scaled) efficiency bound $n\var[\phi(o)]$ is 
$$\rho^{-1}\E[r^2(X)w^2(A,X)\var[Y\mid A,X]]+(1-\rho)^{-1}\var[v(Z)].$$
\end{theorem}

When assuming the model $\cm_{n}^{\mathrm{fix}}$ where $\bpol(a|x)$ and $p(x)$ are fixed at true values, we can also show that the EIF and the efficiency bound are the same. 

\begin{proof}[Proof of \cref{thm:efficiency}]

We follow the following steps.
\begin{enumerate}
    \item Calculate some gradient (a candidate of EIF) of the target functional $R(\epol)$ w.r.t $\cm_n$.
    \item Calculate the tangent space w.r.t $\cm_n$.
    \item Show that the candidate of EIF in Step 1 lies in the tangent space. Then, this concludes that a candidate of EIF in Step 1 is actually the EIF.
\end{enumerate}

\paragraph{Calculation of the gradient}

As mentioned, the model $\cm^{\para}_n$ for a nonparametric model $\cm_n$ is 
\begin{align*}\ts
  p(o;\theta)&=\prod_{i=1}^{\nhs} p(x_i;\theta_x)\bpol(a_i \mid x_i;\theta_a)p(y_i \mid x_i,a_i;\theta_y)\prod_{j=1}^{\nev} q(z_j;\theta_z),\\
  \theta &=(\theta^{\top}_x, \theta^{\top}_a,\theta^{\top}_y,\theta^{\top}_z)^{\top},\, o=\{x_i,a_i,y_i,z_j\}_{i=1,j=1}^{\nhs,\nev}.
\end{align*}
We define the corresponding gradients: 
\begin{align*}
g_{x}=\nabla_{\theta_x}\log p(x;\theta_x), \,g_{a|x}=\nabla_{\theta_a}\log \bpol(a|x;\theta_a),\,g_{y|a,x}=\nabla_{\theta_y}\log p(y|a,x;\theta_y),q_{z}=\nabla_{\theta_z}\log q(z;\theta_z). 
\end{align*}
To derive some gradient of the target functional $R(\epol)$ w.r.t $\cm_n$, what we need is finding a function $f(o)$ satisfying 
\begin{align*}
    \nabla R(\theta)&=\E[f(\cD)\nabla \log p(\cD;\theta)]\\
    &=\E\bracks{ f(\cD) \braces{\frac{1}{\nhs}\sum_{i=1}^{\nhs}\{g_x(X_i)+g_{a|x}(X_i,A_i)+g_{y|x,a}(X_i,A_i,Y_i)\} +\frac{1}{\nev}\sum_{j=1}^{\nev}g_{z}(Z_j) } }. 
\end{align*}
We take the derivative as follows:
\begin{align*}
 \nabla  R(\theta)&= \E_{q(x)\epol(a|x)p(y|a,x)}\bracks{y\braces{g_z(x)+g_{y|a,x}(y|a,x)}}. 
\end{align*}
By some algebra, this is equal to 
\begin{align*}
 \E\bracks{\braces{\frac{1}{\nhs}\sum_{i=1}^{\nhs}r(X_i)w(A_i,X_i)\{Y_i-q(X_i,A_i)\}+
 \frac{1}{\nev}\sum_{j=1}^{\nev}v(Z_j)-R(\pi)} \nabla \log p(\cD;\theta)}.
\end{align*}
Thus, the following function 
\begin{align*}
\phi(o)=\frac{1}{\nhs}\sum_{i=1}^{\nhs}r(x_i)w(a_i,x_i)\{y_i-q(x_i,a_i)\}+
 \frac{1}{\nev}\sum_{j=1}^{\nev}v(z_j)-R(\pi)    
\end{align*}
is a derivative. 

\paragraph{Calculation of the tangent space}

Following a standard derivation way \citep{TsiatisAnastasiosA2006STaM,VaartA.W.vander1998As}, the tangent space of the model $\cm_n$ is 
\begin{align*}
\braces{\frac{1}{\nhs}\sum_{i=1}^{\nhs}\braces{t_{x}(x_i)+t_{a|x}(x_i,a_i)+t_{y|a,x} (x_i,a_i,y_i)}+\frac{1}{\nev}\sum_{j=1}^{\nev} tt_{z}(z_j)\in L_2(o) }. 
\end{align*}
where $L_2(o)$ is an $l_2$ space at the true density, 
\begin{align*}
    \E[t_x(X)]=0,\E[t_{a|x}(X,A)|X]=0,\E[t_{y|a,x}(X,A,Y)|X,A]=0,\E[t_{z}(Z)]=0.
\end{align*}

\paragraph{Last Part}

We can easily check that $\phi(o)$ lies in the tangent space by taking $$t_{x}=0,t_{a|x}=0,t_{y|a,x}=r(x)w(a,x)\{y-q(a,x)\},t_z(z)=v(z)-R(\pi).$$
Thus, $\phi(o)$ is the EIF. 

\begin{remark}
We can easily see that the EIF is $\phi(o)$ when assuming the model $\cm_{n}^{\mathrm{fix}}$ where $p_{x}(x)$ and $\bpol(a|x)$ are fixed at true values. This model is represented as 
\begin{align*}\ts
\braces{\prod_{i=1}^{\nhs} p_{*}(x_i)\bpol_{*}(a_i \mid x_i)p(y_i \mid x_i,a_i;\theta_y)\prod_{j=1}^{\nev} q(z_j;\theta_z)}
\end{align*}
where $\cdot_{*}$ emphasizes that these are fixed at true densities. 
\end{remark}
The function $\phi(o)$ is still a gradient of $R(\epol)$ w.r.t $\cm_{n}^{\mathrm{fix}}$ since the model $\cm_{n}^{\mathrm{fix}}$ is smaller than the model $\cm_n$. Besides, $\phi(o)$ belongs to the tangent spaced induced by the model $\cm_{n}^{\mathrm{fix}}$ since the tangent space induced by $\cm_{n}^{\mathrm{fix}}$  is 
\begin{align*}
\braces{\frac{1}{\nhs}\sum_{i=1}^{\nhs}\braces{t_{y|a,x} (x_i,a_i,y_i)}+\frac{1}{\nev}\sum_{j=1}^{\nev} t_{z}(z_j)\in L_2(o) }
\end{align*}
where 
\begin{align*}
\E[t_{y|a,x}(X,A,Y)|X,A]=0,\E[t_{z}(Z)]=0.
\end{align*}
\end{proof}

\section{Proofs}

\label{sec:proof}
In this section, we show the proofs of theorems. In the proofs of  \cref{thm:main2}--\ref{thm:main1}, we prove the case where we use a two-fold cross-fitting. The extension of two fold cross-fitting to the general $K$-fold cross-fold is straightforward. 

\subsection{Required conditions }

In order to show Theorems~\ref{thm:dm}--\ref{thm:ipw1}, we use the following Theorem~\ref{thm:kernel}, which shows the convergence rate of kernel regression. Here, we have data $\{B_i,C_i\}_{i=1}^{n}$, which are i.i.d. from $p(b,c)=p(c|b)p(b)$, and $B_i$ takes a value in $\mathcal{B}$. Then, let us consider a kernel estimation:
\begin{align*}\ts
    n^{-1}\sum_{i=1}^{n}K_{h}(B_i-b)C_i, 
\end{align*}
where $K_{h}(b)=h^{-d}K(b/h^{d})$, where $d$ is a dimension of $b$. Then, we have the following theorem following \citet{newey94}. 

\begin{theorem} 
\label{thm:kernel}
Assume 
\begin{itemize}
    \item the space $\mathcal{B}$ is compact and $p(b)>0$ on $\mathcal{B}$,
    \item the kernel $K(u)$ has the bounded derivative of order $k$, satisfies $\int K(u)\mathrm{d}u=1$, and has zero moments of order $\leq m-1$ and a nonzero $m$-th order moment,
    \item $\E[C\mid B=b]$ is continuously differentiable to order $k$ with bounded derivatives on the opening set in $\mathcal{B}$. 
    \item there is $v \geq 4$ such that $\rE[|C|^v]\leq \infty$ and $\rE[|C|^{v}\mid B=b]p(b)$ is bounded.  
\end{itemize}
Then, when $h=h(n)$ and $h(n)\to 0$,
\begin{align}\ts
\label{eq:linfty}
\|n^{-1}\sum_{i=1}^{n}K_{h}(B_i-b)C_i-p(b)\E[C|b]\|_{\infty}=\Op\left(\frac{\log n^{1/2}}{(n h^{d+2k})^{1/2}} +h^{m}\right). 
\end{align}
Then, under $n^{1-2/v}h^d/\log n\to \infty,\sqrt{n}h^{d+2k}\to \infty,\sqrt{n}h^{2m}\to 0$, the above $l_{\infty}$ risk is $\op(n^{-1/2})$ \citep{newey94}.
\end{theorem}

\paragraph{Additional assumptions:} regarding Theorem~\ref{thm:kernel}, we assume the following assumptions when we prove Theorems~\ref{thm:ipw1}--\ref{thm:dm}:
\begin{description}
\item [Theorem \ref{thm:ipw1}]: condition when replacing $B$ with $X$ , $C$ with $w(A,X)Y$, condition when replacing $B$ with $Z$, $C$ with $1$. 
    \item [Theorem \ref{thm:ipw2}]: condition when replacing $B$ with $X$ , $C$ with $w(A,X)Y$, condition when replacing $B$ with $X$, $C$ with $1$, condition when replacing $B$ with $Z$, $C$ with $1$. 
    \item [Thorem~\ref{thm:ipw3}]: condition when replacing $B$ with $(X,A)$, $C$ with $1$, condition when replacing $B$ with $(X,A)$, $C$ with $Y$,  condition when replacing $B$ with $X$, $C$ with $w(A,X)Y$, and condition when replacing $B$ with $Z$, $C$ with $1$
    \item [Theorem \ref{thm:dm}]: condition when replacing $B$ with $(X,A)$, $C$ with $Y$, condition when replacing $B$ with $(X,A)$, $C$ with $1$
\end{description}

\subsection{Warming up}
As a warm up, first, we prove the asymptotic property of some simple estimator. When $p(x)$ and $\bpol(a\mid x)$ are known, let us define an IPW estimator:
\begin{align*}
    \hat R_{\ipw1}(\epol)=\E_{\nhs}\left[\frac{\hat q(X)}{p(X)}\frac{\epol(A\mid X)Y}{\bpol(A\mid X)}\right].
\end{align*}

\begin{theorem}
\label{thm:ipw1} When $\hat q(x)=\hat q_h(x)$, the asymptotic MSE of $\hat R_{\ipw1}$ is 
\begin{align*}
    \rho^{-1}\var[r(X)w(A,X)Y]+(1-\rho)^{-1}\var[v(Z)]. 
\end{align*}
\end{theorem}

\begin{proof}[Proof of Theorem~\ref{thm:ipw1}]
We follow the proof of \citet{newey94}. For the ease of notation, assume $\rho=k_1/(k_1+k_2)$.
In this case, $\nhs=k_1\noo$ and $\nev=k_2\noo$, where $\noo=n/(k_1+k_2)$. Note that in this asymptotic regime, $N_o\to \infty$. 
Therefore, we reindex the sample set as
\begin{align*}\ts
    \{X_i\}_{i=1}^{\nhs} &= \{X_{b,i}\}\,(1\leq b \leq k_1,\,1\leq i \leq \noo) ,\\
    \{Z_i\}_{j=1}^{\nhs} &= \{Z_{c,j }\}\,(1\leq c \leq k_2,\,1\leq j \leq \noo).
\end{align*}
Here, we only consider the  estimator  $\hat R_{\ipw1}(\epol)$ based on based on $\{X_{b,i}\}_{i=1}^{\noo}$ and $\{Z_{c,j}\}_{j=1}^{\noo}$, and denote it as $\hat R_{b,c}$. Then, the final estimator $\hat R_{\ipw1}(\epol)$ using all set of samples is equal to  
\begin{align*}\ts
    \frac{1}{k_1k_2}\sum_{b=1}^{k_1}\sum_{c=1}^{k_2}\hat R_{b,c},
\end{align*}
since the kernel estimator has a linear property. More specifically, we have 
\begin{align*}
    \hat R_{\ipw1}(\epol) &=\frac{1}{\nhs}\sum_{i=1}^{\nhs}\left\{\frac{1}{\nev}\sum_{j=1}^{\nev}K_h(Z_j-X_i) \right\}\frac{\epol(A_i\mid X_i)Y_i}{\bpol(A_i \mid X_i)p(X_i)}= \frac{1}{\nhs\nev}\sum_{i=1}^{\nhs}\sum_{j=1}^{\nev}\frac{K_h(Z_j-X_i)\epol(A_i\mid X_i)Y_i}{\bpol(A_i \mid X_i)p(X_i)}\\
    &= \frac{1}{k_1k_2}\sum_{b=1}^{k_1}\sum_{c=1}^{k_2}\left\{\frac{1}{\nhs\nev}\sum_{i=1}^{\nhs}\sum_{j=1}^{\nev}\frac{K_h(Z_{c,j}-X_{b,i})\epol(A_{b,i}\mid X_{b,i})Y_{b,i}}{\bpol(A_{b,i} \mid X_{b,i})p(X_{b,i})} \right\} =    \frac{1}{k_1k_2}\sum_{b=1}^{k_1}\sum_{c=1}^{k_2}\hat R_{b,c}. 
\end{align*}
First, we analyze $\hat R_{1,1}$. 

\paragraph{Step 1}
We prove the following in this step:
\begin{align*}\ts
\hat R_{b,c}=\frac{1}{\noo}\sum_{i=1}^{\noo}r(X_{b,i})w(X_{b,i},A_{b,i})Y_{b,i}+ \frac{1}{\noo}\sum_{j=1}^{\noo}v(Z_{c,j})+\op(n^{-1/2}). 
\end{align*}

Especially, we prove the statement for $\hat R_{1,1}$ when $k_1=1,k_2=1,\nhs=\nev=n/2$. We have 
\begin{align*}\ts
     \hat R_{1,1} &=\E_{\nhs}\left[\frac{\hat q_{h}(X)}{p(X)}\frac{\epol(A\mid X)Y}{\bpol(A\mid X)}\right] \\
     &= \frac{1}{\nhs}\sum_{i=1}^{\nhs}\frac{1}{p(X_i)}\frac{\bpol(A_i\mid X_i)Y_i}{\bpol(A_i\mid X_i)}\left\{\frac{1}{\nev}\sum_{j=1}^{\nev}K_h(Z_j-X_i)\right\} \\ 
     &=\E_{\nhs}\left[\frac{q(X)}{p(X)}\frac{\epol(A\mid X)Y}{\bpol(A\mid X)} \right]+\frac{1}{\nhs \nev}\sum_{i=1}^{\nhs}\sum_{j=1}^{\nev}a_{i,j} \\
          &=\E_{\nhs}\left[\frac{q(X)}{p(X)}\frac{\epol(A\mid X)Y}{\bpol(A\mid X)} \right]+\frac{2}{\nhs \nev}\sum_{i<j} b_{i,j},
\end{align*}
where 
\begin{align*}\ts
    a_{i,j}((X_i,A_i,Y_i),(Z_j)) & = \frac{1}{p(X_i)}\frac{\bpol(A_i\mid X_i)Y_i}{\bpol(A_i\mid X_i)}\{K_h(Z_j-X_i)-q(X_i)\},\\
    b_{i,j}((X_i,A_i,Y_i,Z_i),(X_j,A_j,Y_j,Z_j)) & =0.5\{ a_{i,j}+ a_{j,i}\}. 
\end{align*}
Then, 
\begin{align*}\ts
    &\frac{2}{\nhs \nev}\sum_{i<j}b_{i,j}(X_i,A_i,Y_i,Z_i),(X_j,A_j,Y_j,Z_j)) \\
    &= \frac{2}{\nhs}\left\{\sum_{i=1}^{\nhs}\E[b_{i,j}\mid X_i,A_i,Y_i,Z_i]\right\}+\op(n^{-1/2})\\
    &=\frac{1}{ \nev}\sum_{i=1}^{\nev}\E[a_{j,i}\mid X_i,A_i,Y_i,Z_i]+\frac{1}{ \nhs}\sum_{i=1}^{\nhs}\E[a_{i,j}\mid X_i, A_i, Y_i,Z_i]+\op(n^{-1/2})\\ 
    &=\frac{1}{ \nev}\sum_{i=1}^{\nev}\{v(Z_i)-R(\epol)\}+\op(n^{-1/2}). 
\end{align*}
From the first line to the second line, we used the U-statistics theory \citep[Theorem 12.3]{VaartA.W.vander1998As}. From the third line to the fourth line, based on Theorem \ref{thm:kernel}, we used
\begin{align*}\ts
    \E[a_{j,i}\mid X_i,A_i,Y_i,Z_i]&=\op(n^{-{1/2}})+\E[w(A_i,X_i)Y_i\mid X_i=Z_i]\frac{p(X_i)}{p(X_i)} +\E[r(X_i)w(A_i,X_i)Y_i]\\
    &= \op(n^{-{1/2}})+v(Z_i)-R(\epol), \\
    \E[a_{i,j}\mid X_i,A_i,Y_i,Z_i]&= \op(n^{-{1/2}})+ \frac{1}{p(X_i)}\frac{\bpol(A_i\mid X_i)Y_i}{\bpol(A_i\mid X_i)}\{q(X_i)-q(X_i) \} \\
    &= \op(n^{-{1/2}}).
\end{align*}
\begin{remark}
$\E[h(A_i,X_i,Y_i)\mid X_i=Z_i]$ is an abbreviation of $\{\E[h(A_i,X_i,Y_i)\mid X_i=x]\}_{x=Z_i}$. 
\end{remark}
Therefore, 
\begin{align*}\ts
    \hat R_{\mathrm{IPWCSB}} &=\E_{\nhs}\left[\frac{q(X)}{p(X)}\frac{\epol(A\mid X)Y}{\bpol(A\mid X)} \right]+\E_{\nev}[v(Z)]-R(\epol)+\op(n^{-1/2}). 
\end{align*}

\paragraph{Step 2}

Based on Step 1, we have 
\begin{align*}\ts
   \hat R_{\mathrm{IPW1}} &=\frac{1}{k_1k_2}\sum_{b=1}^{k_1}\sum_{c=1}^{k_2}\hat R_{b,c} \\
   &=    \frac{1}{k_1k_2}\sum_{b=1}^{k_1}\sum_{c=1}^{k_2}\left[\frac{1}{\noo}\sum_{i=1}^{\noo}r(X_{b,i})w(X_{b,i},A_{b,i})Y_{b,i}+ \frac{1}{\noo}\sum_{j=1}^{\noo}\{v(Z_{c,j})\}\right]-R(\epol)+\op(n^{-1/2})\\
    &=    \frac{1}{k_1\noo}\sum_{b=1}^{k_1}\left[\sum_{i=1}^{\noo}r(X_{b,i})w(X_{b,i},A_{b,i})Y_{b,i}\right]+ \frac{1}{k_2\noo}\sum_{c=1}^{k_2}\left[ \sum_{j=1}^{\noo}v(Z_{c,j})\right]-R(\epol)+\op(n^{-1/2}) \\ 
     &=    \frac{1}{\nhs}\sum_{i=1}^{\nhs}r(X_{i})w(X_{i},A_{i})Y_{i}+ \frac{1}{\nev}\sum_{j=1}^{\nev} v(Z_{j})-R(\epol)+\op(n^{-1/2}). 
\end{align*}
Finally, from stratified sampling CLT, the statement is concluded. 
\end{proof}

\subsection{Proof of Theorem~\ref{thm:main2}}
\begin{proof}
We denote 
$$\phi_1(x,a,y;r,w,f)=r(x)w(a,x)\{y-f(a,x)\},\,\phi_2(z;f)=v(z).$$
We also denote the union of $\nhs_i$ and $\nev_i$ as $\cD_i$ for $i=1,2$, and the number of $\nhs_1,\nhs_2,\nev_1,\nev_2$ as $n_{11},n_{21},n_{12},n_{22}$. 
For simplicity, we assume $n_{11}=n_{12},n_{21}=n_{22}$. 

Then, we have 
\begin{align}\ts
   &\sqrt{n}\{\E_{\nhs_1}[\phi_1(X,A,Y;\hat r^{(1)},\hat w^{(1)},\hat f^{(1)})]+\E_{\nev_1}[\phi_2(Z;\hat f^{(1)})]-R(\epol)\} \nonumber \\ &=\sqrt{n}\left\{\frac{1}{\sqrt{n_{11}}}\bG_{\nev_1}[\phi_1(X,A,Y;\hat r^{(1)},\hat w^{(1)},\hat f^{(1)})-\phi_1(X,A,Y;r,w, f)]+\frac{1}{\sqrt{n_{12}}}\bG_{\nev_1}[\phi_2(Z;\hat f^{(1)})-\phi_2(Z; f)]\right\}  \label{eq:term11}\\
   &+\sqrt{n}\{\E[\phi_1(X,A,Y;\hat r^{(1)},\hat w^{(1)},\hat f^{(1)})\mid \hat r^{(1)},\hat w^{(1)},\hat f^{(1)}]+\E[\phi_2(Z;\hat f^{(1)})\mid \hat f^{(1)}]\label{eq:term12} \\
   &-\E[\phi_1(X,A,Y;r,w, f)]-\E[\phi_2(Z;f)]\} \nonumber \\
   &+\sqrt{n}\{\E_{_{\nhs_1}}[\phi_1(X,A,Y;r,w, f)]+\E_{\nev_1}[\phi_2(Z;f)]-R(\epol)\}\label{eq:term13}. 
\end{align}
The  term \eqref{eq:term11} is $\op(1)$ by Step 1. The term \cref{eq:term12} is also $\op(1)$ by Step 2 as follows.

\paragraph{Step 1:}\cref{eq:term11} is $\op(1)$. 

If we can show that for any $\epsilon>0$, 
\begin{align}\ts
\label{eq:part}
& \lim_{n\to \infty}P[|\sqrt{n}\{\frac{1}{\sqrt{n_{11}}}\bG_{\nhs_1}[\phi_1(X,A,Y;\hat r^{(1)},\hat w^{(1)},\hat f^{(1)})-\phi_1(X,A,Y;r,w, f)] \\
&\ \ \ \ +\frac{1}{\sqrt{n_{12}}}\bG_{\nev_1}[\phi_2(Z;\hat f^{(1)})-\phi_2(Z; f)]\} |>\epsilon \mid \cD_{2}]=0, \nonumber
\end{align}
then by the bounded convergence theorem, we would have 
\begin{align*}\ts
&\lim_{n\to \infty}P[|\sqrt{n}\{\frac{1}{\sqrt{n_{11}}}\bG_{\nhs_1}[\phi_1(X,A,Y;\hat r^{(1)},\hat w^{(1)},\hat f^{(1)})-\phi_1(X,A,Y;r,w, f)]\\
&\ \ \ +\frac{1}{\sqrt{n_{12}}}\bG_{\nev_1}[\phi_2(Z;\hat f^{(1)})-\phi_2(Z; f)]\} |>\epsilon ]=0,
\end{align*}
yielding the statement. 

To show \eqref{eq:part}, we show that the conditional mean is $0$ and the conditional variance is $\op(1)$. Then, \eqref{eq:part} is proved by the Chebyshev inequality following the proof of \citep[Theorem 4]{KallusUehara2019}. 
The conditional mean is 
\begin{align*}\ts
&\E[\sqrt{n}\{\frac{1}{\sqrt{n_{11}}}\bG_{\nhs_1}[\phi_1(X,A,Y;\hat r^{(1)},\hat w^{(1)},\hat f^{(1)})-\phi_1(X,A,Y;r,w, f)] \\ 
&\ \ \ \ +\frac{1}{\sqrt{n_{12}}}\bG_{\nev_1}[\phi_2(Z;\hat f^{(1)})-\phi_2(Z; f)]\} \mid \cD_{2}] \\
&= \E[\sqrt{n}\{\frac{1}{\sqrt{n_{11}}}\bG_{\nhs_1}[\phi_1(X,A,Y;\hat r^{(1)},\hat w^{(1)},\hat f^{(1)})-\phi_1(X,A,Y;r,w, f)] \\ 
&\ \ \ \ +\frac{1}{\sqrt{n_{12}}}\bG_{\nev_1}[\phi_2(Z;\hat f^{(1)})-\phi_2(Z; f)]\} \mid \cD_{2}, \hat r^{(1)},\hat w^{(1)},\hat f^{(1)}] \\
&=0. 
\end{align*}
Here, we used a cross-fitting construction. More specifically, regarding the second term, we have
\begin{align*}\ts
    &\E[\E_{\nev_1}[\phi_2(Z;\hat f^{(1)})-\phi_2(Z; f)]- \E[\phi_2(Z;\hat f^{(1)})-\phi_2(Z; f)]\mid \cD_{2}, \hat r^{(1)},\hat w^{(1)},\hat f^{(1)} ]\\
    &=\E[\E_{\nev_1}[\phi_2(Z;\hat f^{(1)})-\phi_2(Z; f)] \mid \cD_{2}, \hat r^{(1)},\hat w^{(1)},\hat f^{(1)}] -\E[\phi_2(Z;\hat f^{(1)})-\phi_2(Z; f)]\mid \hat f^{(1)}]\\
    &=\E[\phi_2(Z;\hat f^{(1)})-\phi_2(Z; f)\mid \hat f^{(1)}]-\E[\phi_2(Z;\hat f^{(1)})-\phi_2(Z; f)\mid \hat f^{(1)}]=0. 
\end{align*}
The conditional variance is bounded as 
\begin{align*}\ts
&\var[\sqrt{n}\{\frac{1}{\sqrt{n_{11}}}\bG_{n_{11}}[\phi_1(X,A,Y;\hat r^{(1)},\hat w^{(1)},\hat f^{(1)})-\phi_1(X,A,Y;r,w, f)\mid \cD_{2}]\\
&+\frac{1}{\sqrt{n_{12}}}\bG_{n_{12}}[\phi_2(Z;\hat f^{(1)})-\phi_2(Z; f)]\} \mid \cD_{2}] \\
&= \frac{n}{n_{11}}\var[\phi_1(X,A,Y;\hat r^{(1)},\hat w^{(1)},\hat f^{(1)})-\phi_1(X,A,Y;r,w, f) \mid \cD_{2}] \\
&+\frac{n}{n_{22}}\var[\phi_2(Z;\hat f^{(1)})-\phi_2(Z; f)\mid \cD_{2} ] \\
&\leq \frac{n}{n_{11}}\E[\{\hat r^{(1)}(X) \hat w^{(1)}(A,X)(Y-\hat f^{(1)}(A,X))-r(X)w(A,X)(Y-f(A,X))\}^2 \mid \cD_{2} ]\\
&+\frac{n}{n_{22}}\E[\{ \hat v^{(1)}(Z)-v(Z)\}^2 \mid \cD_{2} ]=\op(1)+\op(1)=\op(1). 
\end{align*}

Here, we used
\begin{align}\ts
\label{eq:first_e}
    \frac{n}{n_{11}}\E[\{\hat r^{(1)}(X) \hat w^{(1)}(A,X)(Y-\hat f^{(1)}(A,X))-r(X)w(A,X)(Y-f(A,X))\}^2 \mid \cD_{2} ]=\op(1). 
\end{align}
and 
\begin{align}\ts
\label{eq:second_e}
    \E[\{ \hat v^{(1)}(Z)-v(Z)\}^2 \mid \cD_{2} ]=\op(1). 
\end{align}
The first equation \eqref{eq:first_e} is proved by 
\begin{align*}\ts
    &\E[\{\hat r^{(1)} \hat w^{(1)}(Y-\hat f^{(1)})-rw(Y-f)\}^2 \mid \cD_{2} ]\\
    &=\E[\{\hat r^{(1)} \hat w^{(1)}(Y-\hat f^{(1)})-\hat r^{(1)} \hat w^{(1)}(Y-f) + \hat r^{(1)} \hat w^{(1)}(Y-f) -rw(Y-f)\}^2 \mid \cD_{2} ] \\
    &\leq 2\E[\{\hat r^{(1)} \hat w^{(1)}(Y-\hat f^{(1)})-\hat r^{(1)} \hat w^{(1)}(Y-f)\}^2 \mid \cD_{2} ]+2\E[ \{\hat r^{(1)} \hat w^{(1)}(Y-f) -rw(Y-f)\}^2\mid \cD_{2} ] \\
    &\leq 2C_1 C_2 \|f-\hat f^{(1)}\|^2_2+ 2\times 4R^2_{\max}\|\hat r^{(1)}\hat w^{(1)}-rw\|^2_2=\op(1). 
\end{align*}
Here, we have used a parallelogram law from the second line to the third line. We have use $0<\hat r<C_1,0<\hat w<C_2,|\hat f|<R_{\max}$ according to the Assumption \ref{asm:global2} and convergence rate conditions, from the third line to the fourth line.  The second equation \eqref{eq:second_e} is proved by Jensen's inequality. 

\paragraph{Step 2:}\cref{eq:term12} is $\op(1)$. 

We have 
\begin{align*}\ts
    & |\E[\phi_1(X,A,Y;\hat r^{(1)},\hat w^{(1)},\hat f^{(1)})\mid \hat r^{(1)},\hat w^{(1)},\hat f]+\E[\phi_2(Z;\hat f^{(1)})\mid \hat f]-\E[\phi_1(x;r,w, f)]-\E[\phi_2(Z;f)]|\\
    &\leq |\E[\{\hat r^{(1)}(X)\hat w^{(1)}(A,X)-r(X)w(A,X) \}\{-\hat f^{(1)}(A,X)+f(A,X) \} \mid \hat r^{(1)},\hat w^{(1)},\hat f^{(1)}]|\\
    &+|\E[r(x)w(A,X)\{-\hat f(A,X)+f(A,X)\}\mid \hat r^{(1)},\hat w^{(1)},\hat f^{(1)} ]+\E[\hat v^{(1)}(Z)-v(Z)\mid \hat f^{(1)}]|\\
    &+|\E[\hat r^{(1)}(X)\hat w^{(1)}(A,X)\{Y-f(A,X)\}\mid \hat r^{(1)},\hat w^{(1)} ]|\\
    &\leq \|\hat r^{(1)}(X)\hat w^{(1)}(A,X)- r(X)w(A,X)\|_2\|\hat f^{(1)}(A,X)-f(A,X)\|_2+0+0 \\
    &=\alpha\beta+0+0=\op(n^{-1/2}).
\end{align*}
Here, we have used \Holder's inequality:
$$ \|fg \|_1 \leq  \|f \|_2  \|g \|_2,$$
the relation
\begin{align*}\ts
    &\E[r(X)w(A,X)\{-\hat f^{(1)}(A,X)+f(A,X)\}\mid \hat r^{(1)},\hat w^{(1)},\hat f^{(1)} ]+\E[\hat v^{(1)}(Z)-v(Z)\mid \hat f^{(1)}]\\ 
    &=\E[-r(X)w(A,X)\hat f^{(1)}(A,X)+\hat v^{(1)}(z) \mid \hat r^{(1)},\hat w^{(1)},\hat f^{(1)}]+\E[-r(X)w(A,X)f(A,X)+v(Z)] \\
    &=0+0=0,
\end{align*}
and 
\begin{align*}\ts
    &\E[\hat r^{(1)}(X)\hat w^{(1)}(A,X)\{Y-f(A,X)\}\mid \hat r^{(1)},\hat w^{(1)} ]\\
    &=\E[\hat r^{(1)}(X)\hat w^{(1)}(A,X)\{f(A,X)-f(A,X)\}\mid \hat r^{(1)},\hat w^{(1)} ]=0. 
\end{align*}

\paragraph{Step 3:} By combining everything, we have 
\begin{align*}\ts
&\mathbb{E}_{\nhs_1}[\phi_1(X,A,Y;\hat r^{(1)},\hat w^{(1)},\hat f^{(1)})]+\mathbb{E}_{\nev_1}[\phi_2(Z;\hat f^{(1)})]-R(\epol) \\  
    &=\mathbb{E}_{\nhs_1}[\phi_1(X,A,Y;r,w, f)]+\mathbb{E}_{\nev_1}[\phi_2(Z;f)]-R(\epol)+\op(1/\sqrt{n}). 
\end{align*}
Then, 
\begin{align*}\ts
\hat R_{\dml}&= 0.5\mathbb{E}_{\nhs_1}[\phi_1(X,A,Y;\hat r^{(1)},\hat w^{(1)},\hat f^{(1)})]+0.5\mathbb{E}_{\nev_1}[\phi_2(Z;\hat f^{(1)})]\\
&\ \ \ \ +0.5\mathbb{E}_{\nhs_2}[\phi_1(X,A,Y;\hat r^{(2)},\hat w^{(2)},\hat f^{(2)})]+0.5\mathbb{E}_{\nev_2}[\phi_2(Z;\hat f^{(2)})] \\
    &=0.5\mathbb{E}_{\nhs_1}[\phi_1(X,A,Y;r,w, f)]+0.5\mathbb{E}_{\nev_1}[\phi_2(Z;f)]+ \\
    &\ \ \ \ +0.5\mathbb{E}_{\nhs_2}[\phi_1(X,A,Y;r,w,f)]+0.5\mathbb{E}_{\nev_2}[\phi_2(Z;f)]+\op(1/\sqrt{n}) \\
    &=\mathbb{E}_{\nhs}[\phi_1(X,A,Y;r,w, f)]+\mathbb{E}_{\nev}[\phi_2(Z;f)]+ \op(1/\sqrt{n}). 
\end{align*}
Finally, by using a stratified sampling CLT \citep{WooldridgeJeffreyM.2001APOW}, the statement is concluded based on Assumption \ref{asm:global}. 
\end{proof}

\subsection{Proof of Lemma~\ref{lem:den}}
\begin{proof}
We can bound $\|\hat r(X)\hat w(A,X)-r(x)w(A,X)\|_2=\op(n^{-p})$:
\begin{align*}\ts
     \| \hat r(X)\hat w(A,X)-r(X)w(A,X)\|_2 &\leq \| \hat r(X)\hat w(A,X)-\hat r(X)w(A,X)\|_2\\
     &\ \ \ + \| \hat r(X)w(A,X)-r(X)w(A,X)\|_2  \\
     &\leq C_1\op(n^{-p})+C_2\op(n^{-p})=\op(n^{-p}). 
\end{align*}
Here, we used the assumptions that $r(X)$ is uniformly bounded by $C_1$ and $w(A,X)$ is uniformly bounded by $C_2$. 
\end{proof}

\subsection{Proof of Theorem~\ref{thm:model}}
\label{appdx:thm:model}

\begin{proof}
Let us define $\phi_1(x,a,y;r,w,f)$ and $\phi_2(z;f)$:
\begin{align}\ts
\phi_1(x,a,y;r,w,f)=r(x)w(a,x)\{y-f(a,x)\},\ \phi_2(z;f)=v(z).
\end{align}
We also denote the union of $\nhs_i$ and $\nev_i$ by $\cD_i$ for $i=1,2$, and the number of $\nhs_1,\nhs_2,\nev_1,\nev_2$ by $n_{11},n_{21},n_{12},n_{22}$. 
For simplicity, we assume $n_{11}=n_{12},n_{21}=n_{22}$. 

Then, we have 
\begin{align}\ts
   &\{\E_{\nhs_1}[\phi_1(X,A,Y;\hat r^{(1)},\hat w^{(1)},\hat f^{(1)})]+\E_{\nev_1}[\phi_2(Z;\hat f^{(1)})]-R(\epol)\} \nonumber \\ &=\left\{\frac{1}{\sqrt{n_{11}}}\bG_{\nhs_1}[\phi_1(X,A,Y;\hat r^{(1)},\hat w^{(1)},\hat f^{(1)})-\phi_1(X,A,Y;r^{\dagger},w^{\dagger}, f^{\dagger})]+\frac{1}{\sqrt{n_{12}}}\bG_{\nev_1}[\phi_2(Z;\hat f^{(1)})-\phi_2(Z; f^{\dagger})]\right\}+  \label{eq:term11_model}\\
   &+\{\E[\phi_1(X,A,Y;\hat r^{(1)},\hat w^{(1)},\hat f^{(1)})\mid \hat r^{(1)},\hat w^{(1)},\hat f^{(1)}]+\E[\phi_2(Z;\hat f^{(1)})\mid \hat f^{(1)}]\label{eq:term12_model} \\
   &-\E[\phi_1(X,A,Y;r^{\dagger},w^{\dagger}, f^{\dagger})]-\E[\phi_2(Z;f^{\dagger})]\} \nonumber \\
   &+\{\E_{_{\nhs_1}}[\phi_1(X,A,Y;r^{\dagger},w^{\dagger}, f^{\dagger})]+\E_{\nev_1}[\phi_2(Z;f^{\dagger})]-R(\epol)\}\label{eq:term13_model}. 
\end{align}
The  term \eqref{eq:term11_model} is $\op(1/\sqrt{n})$ by Step 1 in the previous theorem noting that what we have used is $\|\hat r(X)\hat w(A,X)-w^{\dagger}(A,X)r^{\dagger}(X)\|=\op(1),\,\|\hat f(A,X)-f^{\dagger}(A,X)\|=\op(1)$. The term \cref{eq:term12_model} is also $\op(1)$ by Step 1 as we will show soon.

\paragraph{Step 1:} \cref{eq:term12_model} is $\op(1)$. 
We have 
\begin{align*}\ts
    & |\E[\phi_1(X,A,Y;\hat r^{(1)},\hat w^{(1)},\hat f^{(1)})\mid \hat r^{(1)},\hat w^{(1)},\hat f^{(1)}]+\E[\phi_2(Z;\hat f^{(1)})\mid \hat f^{(1)}]-\E[\phi_1(x;r,w, f)]-\E[\phi_2(Z;f)]|\\
    &\leq |\E[\{\hat r^{(1)}(X)\hat w^{(1)}(A,X)-r^{\dagger}(X)w^{\dagger}(A,X) \}\{-\hat f^{(1)}(A,X)+f^{\dagger}(A,X) \} \mid \hat r^{(1)},\hat w^{(1)},\hat f^{(1)}]|\\
    &+|\E[r^{\dagger}(x)w^{\dagger}(A,X)\{-\hat f^{(1)}(A,X)+f^{\dagger}(A,X)\}\mid \hat r^{(1)},\hat w^{(1)},\hat f^{(1)} ]+\E[\hat v^{(1)}(Z)-v^{\dagger}(Z)\mid \hat f^{(1)}]|\\
    &+|\E[\hat r^{(1)}(X)\hat w^{(1)}(A,X)\{Y-f^{\dagger}(A,X)\}\mid \hat r^{(1)},\hat w^{(1)} ]. 
\end{align*}
Here, if $f^{\dagger}(a,x)=f(a,x)$, we have 
\begin{align*}\ts
    \op(1)\op(1)+\op(1)+0=\op(1)=\op(1). 
\end{align*}
if $r^{\dagger}(x)w^{\dagger}(a,x)=r(x)w(a,x)$, we have 
\begin{align*}\ts
    \op(1)\op(1)+0+\op(1)=\op(1)=\op(1). 
\end{align*}
Therefore, \cref{eq:term12_model} is $\op(1)$. 

\paragraph{Step 2:}

By combining togather, we have 
\begin{align*}\ts
&\mathbb{E}_{\nhs_1}[\phi_1(X,A,Y;\hat r^{(1)},\hat w^{(1)},\hat f^{(1)})]+\mathbb{E}_{\nev_1}[\phi_2(Z;\hat f^{(1)})]-R(\epol) \\  
    &=\mathbb{E}_{\nhs_1}[\phi_1(X,A,Y;r^{\dagger},w^{\dagger}, f^{\dagger})]+\mathbb{E}_{\nev_1}[\phi_2(Z;f^{\dagger})]-R(\epol)+\op(1). 
\end{align*}
Then, 
\begin{align*}\ts
\hat R_{\dml}&= 0.5\mathbb{E}_{\nhs_1}[\phi_1(X,A,Y;\hat r^{(1)},\hat w^{(1)},\hat f^{(1)})]+0.5\mathbb{E}_{\nev_1}[\phi_2(Z;\hat f^{(1)})]\\
&\ \ \ \ +0.5\mathbb{E}_{\nhs_2}[\phi_1(X,A,Y;\hat r^{(2)},\hat w^{(2)},\hat f^{(2)})]+0.5\mathbb{E}_{\nev_2}[\phi_2(Z;\hat f^{(2)})] \\
    &=0.5\mathbb{E}_{\nhs_1}[\phi_1(X,A,Y;r^{\dagger},w^{\dagger}, f^{\dagger})]+0.5\mathbb{E}_{\nev_1}[\phi_2(Z;f^{\dagger})]+ \\
    &\ \ \ \ +0.5\mathbb{E}_{\nhs_2}[\phi_1(X,A,Y;r^{\dagger},w^{\dagger},f^{\dagger})]+0.5\mathbb{E}_{\nev_2}[\phi_2(Z;f^{\dagger})]+\op(1) \\
    &=\mathbb{E}_{\nhs}[\phi_1(X,A,Y;r^{\dagger},w^{\dagger}, f^{\dagger})]+\mathbb{E}_{\nev}[\phi_2(Z;f^{\dagger})]+ \op(1). 
\end{align*}
Then, the statement is concluded since 
\begin{align*}\ts
  \E\left[\mathbb{E}_{\nhs}[\phi_1(X,A,Y;r^{\dagger},w^{\dagger}, f^{\dagger})]+\mathbb{E}_{\nev}[\phi_2(Z;f^{\dagger})]\right]=R(\epol)
\end{align*}
based on the double robust structure  and 
\begin{align*}\ts
    \mathbb{E}_{\nhs}[\phi_1(X,A,Y;r^{\dagger},w^{\dagger}, f^{\dagger})]+\mathbb{E}_{\nev}[\phi_2(Z;f^{\dagger})]=R(\epol)+ \op(1)
\end{align*}
from the law of large numbers based on Assumption \ref{asm:global}. 
\end{proof}

\subsection{Proof of Theorem \ref{thm:main1}}
\begin{proof}

We can prove similarly as in the proof of \cref{thm:efficiency}. Therefore, we omit the proof.  

\end{proof}

\subsection{Proof of Theorem~\ref{thm:ipw2}}
\begin{proof}[Proof of Theorem~\ref{thm:ipw2}]
For the ease of notation, we prove the case $\nhs=\nev$ noting the kernel estimator is linearized as in Theorem \ref{thm:ipw1} and and the generalization is easy. We have
\begin{align*}\ts
   \left \|\frac{\hat q_{h}(x)}{\hat p_{h}(x)}-\frac{ q(x)}{p(x)}-\hat e_{h}(x)\right\|_{\infty}= \op(n^{-1/2}), 
\end{align*}
where 
\begin{align*}\ts
    \hat e_{h}(x)= \frac{1}{p(x)}\{\hat q_{h}(x)-q(x) \}-\frac{q(x)}{p^2(x)}\{\hat p_{h}(x)-p(x) \}. 
\end{align*}
This is proved by Theorem \ref{thm:kernel}. Then, 
\begin{align*}\ts
     \hat R_{\mathrm{IPWCSB}} &=\E_{\nhs}\left[\frac{\hat q_{h}(X)}{\hat p_h(X)}\frac{\epol(A\mid X)Y}{\bpol(A\mid X)}\right] \\
     &= \frac{1}{\nhs}\sum_{i=1}^{\nhs}\frac{\epol(A_i\mid X_i)Y_i}{\bpol(A_i\mid X_i)}\left\{r(X_i)+\hat e_{h}(X_i) \right\} \\ 
     &=\E_{\nhs}\left[\frac{q(X)}{p(X)}\frac{\epol(A\mid X)Y}{\bpol(A\mid X)} \right]+\frac{1}{\nhs \nev}\sum_{i=1}^{\nhs}\sum_{j=1}^{\nev}a_{i,j} \\
          &=\E_{\nhs}\left[\frac{q(X)}{p(X)}\frac{\epol(A\mid X)Y}{\bpol(A\mid X)} \right]+\frac{2}{\nhs \nev}\sum_{i<j} b_{i,j},
\end{align*}
where 
\begin{align*}\ts
    a_{i,j}((X_i,A_i,Y_i),(Z_j)) & = \frac{1}{p(X_i)}\frac{\epol(A_i\mid X_i)Y_i}{\bpol(A_i\mid X_i)}\{K_h(Z_j-X_i)-q(X_i)\}\\
    &-\frac{q(X_i)}{p^2(X_i)}\frac{\epol(A_i\mid X_i)Y_i}{\bpol(A_i\mid X_i)}\{K_h(X_j-X_i)-p(X_i)\},\\
    b_{i,j}((X_i,A_i,Y_i,Z_i),(X_j,A_j,Y_j,Z_j)) & =0.5\{ a_{i,j}+ a_{j,i}\}. 
\end{align*}
Then, 
\begin{align*}\ts
    &\frac{2}{\nhs \nev}\sum_{i<j}b_{i,j}(X_i,A_i,Y_i,Z_i),(X_j,A_j,Y_j,Z_j)) \\
    &= \frac{2}{\nhs}\left\{\sum_{i=1}^{\nhs}\E[b_{i,j}\mid X_i,A_i,Y_i,Z_i]\right\}+\op(n^{-1/2})\\
    &=\frac{1}{ \nev}\sum_{i=1}^{\nev}\E[a_{j,i}\mid X_i, A_i, Y_i,Z_i]+\frac{1}{ \nhs}\sum_{i=1}^{\nhs}\E[a_{i,j}\mid X_i, A_i, Y_i,Z_i]+\op(n^{-1/2})\\ 
    &=\frac{1}{ \nev}\sum_{i=1}^{\nev}\{v(Z_i)-r(X_i)v(X_i)\}+\op(n^{-1/2}). 
\end{align*}
From the first line to the second line, we have used a U-statistics theory \citep[Chapter 12]{VaartA.W.vander1998As}. From the third line to the fourth line, we have used
\begin{align*}\ts
    &\E[a_{j,i}\mid Z_i,X_i,A_i,Y_i] \\
    &=\op(n^{-1/2})+\left\{\E \left[\frac{1}{p(X_i)}\frac{\epol(A_i\mid X_i)Y_i}{\bpol(A_i\mid X_i)}\mid  X_i=Z_i\right]-\E\left[\frac{q(X_i)}{p^2(X_i)}\frac{\epol(A_i\mid X_i)Y_i}{\bpol(A_i\mid X_i)} \mid X_i\right]\right\}p(X_i) \\
    &= \op(n^{-1/2})+v(Z_i)-r(X_i)v(X_i), \\
    &\E[a_{i,j}\mid Z_i,X_i,A_i,Y_i] = \op(n^{-1/2}).
\end{align*}
Therefore, 
\begin{align*}\ts
    \hat R_{\mathrm{IPWCSB}} &=\E_{\nhs}\left[\frac{q(X)}{p(X)}\left\{\frac{\epol(A\mid X)Y}{\bpol(A\mid X)}-v(X)\right\}\right]+\E_{\nev}[v(Z)]+\op(n^{-1/2}). 
\end{align*}
The final statement is concluded by CLT.
\end{proof}

\subsection{Proof of Theorem~\ref{thm:ipw3}}
\begin{proof}\label{appdx:thm:ipw3}
For the ease of the notation, we prove the case $\nhs=\nev$ noting the kernel estimator is linearized as in Theorem \ref{thm:ipw1} and the generalization is easy. Here, we concatenate $X$ and $A$ as $D$. We also write $p(x)\bpol(a\mid x)$ as $u(d)$. 
\begin{align*}\ts
   \left \|\frac{\hat q_{h}(x)}{\hat u_{h}(d)}-\frac{ q(x)}{u(d)}-\hat e_{h}(d)\right\|_{\infty}= \op(n^{-1/2}),
\end{align*}
where 
\begin{align*}\ts
    \hat e_{h}(d)= \frac{1}{u(d)}\{\hat q_{h}(x)-q(x) \}-\frac{q(x)}{u^2(d)}\{\hat u_{h}(d)-u(d) \}. 
\end{align*}
This is proved by Theorem \ref{thm:kernel}. Then, 
\begin{align*}\ts
     \hat R_{\mathrm{IPWCS}} &=\E_{\nhs}\left[\frac{\hat q_{h}(X)}{\hat p_{h}(X)}\frac{\epol(A\mid X)Y}{\hat{\pi}^{b}_h(A\mid X)}\right] \\
     &= \frac{1}{\nhs}\sum_{i=1}^{\nhs}\epol(A_i\mid X_i)Y_i\left\{\frac{q(X_i)}{u(D_i)}+\hat e_{h}(X_i) \right\} \\ 
     &=\E_{\nhs}\left[\frac{q(X)}{p(X)}\frac{\epol(A\mid X)Y}{\bpol(A\mid X)} \right]+\frac{1}{\nhs \nev}\sum_{i=1}^{\nhs}\sum_{j=1}^{\nev}a_{i,j} \\
          &=\E_{\nhs}\left[\frac{q(X)}{p(X)}\frac{\epol(A\mid X)Y}{\bpol(A\mid X)} \right]+\frac{2}{\nhs \nev}\sum_{i<j} b_{i,j},
\end{align*}
where 
\begin{align*}\ts
    a_{i,j}((X_i,A_i,Y_i),(Z_j)) & = \frac{1}{p(X_i)}\frac{\epol(A_i\mid X_i)Y_i}{\bpol(A_i\mid X_i)}\{K_h(Z_j-X_i)-q(X_i)\}\\
    &-\frac{q(X_i)}{u^2(D_i)}\epol(A_i\mid X_i)Y_i\{K_h(D_j-D_i)-u(D_i)\},\\
    b_{i,j}((X_i,A_i,Y_i,Z_i),(X_j,A_j,Y_j,Z_j)) & =0.5\{ a_{i,j}+ a_{j,i}\}. 
\end{align*}
Then, 
\begin{align*}\ts
    &\frac{2}{\nhs \nev}\sum_{i<j}b_{i,j}((X_i,A_i,Y_i,Z_i),(X_j,A_j,Y_j,Z_j)) \\
    &= \frac{2}{\nhs}\left\{\sum_{i=1}^{\nhs}\E[b_{i,j}\mid X_i,A_i,Y_i,Z_i]\right\}+\op(n^{-1/2})\\
    &=\frac{1}{ \nev}\sum_{i=1}^{\nev}\E[a_{j,i}\mid X_i, A_i, Y_i,Z_i]+\frac{1}{ \nhs}\sum_{i=1}^{\nhs}\E[a_{i,j}\mid X_i, A_i, Y_i,Z_i]+\op(n^{-1/2})\\ 
    &=\frac{1}{ \nev}\sum_{i=1}^{\nev}\left\{v(Z_i)-r(X_i)w(X_i,A_i)f(D_i) \right\}+\op(n^{-1/2}). 
\end{align*}
From the first line to the second line, we used the U-statistics theory \citep[Chapter 12]{VaartA.W.vander1998As}. From the third line to the fourth line, we used
\begin{align*}\ts
    &\E[a_{j,i}\mid Z_i,X_i,A_i,Y_i] \\
    &= \op(n^{-1/2})+\E\left[\frac{1}{p(X_i)}w(A_i,X_i)Y_i\mid X_i=Z_i\right]p(X_i)-\E\left[\frac{q(X_i)}{u^2(D_i)}\epol(A_i\mid X_i)Y_i\mid D_i=D_i \right]u(D_i)  \\
    &= \op(n^{-1/2})+v(Z_i)-r(X_i)w(X_i,A_i)f(D_i), \\
    &\E[a_{i,j}\mid Z_i,X_i,A_i,Y_i] =  \op(n^{-1/2}). 
\end{align*}
Therefore, 
\begin{align*}\ts
    \hat R_{\mathrm{IPWCS}} &=\E_{\nhs}\left[\frac{q(X)}{p(X)}\frac{\epol(A\mid X)}{\bpol(A\mid X)}\{Y-f(A,X)\}\right]+\E_{\nev}[v(Z)]+\op(n^{-1/2}). 
\end{align*}
The final statement is concluded by CLT. 
\end{proof}

\subsection{Proof of Theorem~\ref{thm:dm}}
\begin{proof}
For the ease of the notation, we prove the case $\nhs=\nev$ noting the kernel estimator is linearized as in Theorem \ref{thm:ipw1} and and generalization is easy.
Here, 
$\hat v_{h}(a,x)$ is defined as 
\begin{align*}\ts
    \hat v_{h}(a,x)&=\frac{\hat p_{h}(a,x)}{\hat u_{h}(a,x)},\,\hat p_{h}(a,x)=\frac{1}{\nhs}\sum_{i=1}^{\nhs}Y_i K_{h}(\{X_i,A_i\}-\{x,a\}) \\
    \hat u_{h}(a,x)&=\frac{1}{\nhs}\sum_{i=1}^{\nhs} K_{h}(\{X_i,A_i\}-\{x,a\}). 
\end{align*}
We have 
\begin{align*}\ts
   \left \|\frac{\hat p_{h}(a,x)}{\hat u_{h}(a,x)}-\frac{f(a,x)u(a,x) }{u(a,x)}-\hat e_{h}(x,a)\right\|_{\infty}= \op(n^{-1/2}), 
\end{align*}
where 
\begin{align*}\ts
    \hat e_{h}(x,a)= \frac{1}{u(a,x)}\{\hat p_{h}(a,x)-f(a,x)u(a,x) \}-\frac{f(a,x)}{u(a,x)}\{\hat u_{h}(a,x)-u(a,x) \}. 
\end{align*}
This is proved by Theorem \ref{thm:kernel}. Then, we have 
\begin{align*}\ts
    \hat R_{\dm} &=\E_{\nev}[\hat v_h(Z,A)] \\ 
    &= \E_{\nev}[v(Z,A)]+\frac{1}{\nhs \nev}\sum_{i=1}^{\nev}\sum_{j=1}^{\nhs}a_{i,j}\\
 &=\E_{\nev}[v(Z,A)]+\frac{2}{\nhs \nev}\sum_{i<j} b_{i,j},
\end{align*}
where 
\begin{align*}\ts
    a_{i,j}((Z_i,A_i),(X_j,A_j,Y_j)) 
    &= \frac{1}{u(Z_i,A_i)}\left\{Y_j  K_h(\{X_j,A_j\}-\{Z_i,A_i\})-u(Z_i,A_i)f(Z_i,A_i) \right\}-\\
    &\frac{f(Z_i,A_i)}{u(Z_i,A_i)}\{K_h(\{X_j,A_j\}-\{Z_i,A_i\})-u(Z_i,A_i)\},\\
    b_{i,j}((X_i,A_i,Y_i,Z_i),(X_j,A_j,Y_j,Z_j)) &=0.5\{ a_{i,j}+ a_{j,i}\}. 
\end{align*}
Then, 
\begin{align*}\ts
    &\frac{2}{\nhs \nev}\sum_{i<j}b_{i,j}(X_i,A_i,Y_i,Z_i),(X_j,A_j,Y_j,Z_j)) \\
    &= \frac{2}{\nhs}\left\{\sum_{i=1}^{\nhs}\E[b_{i,j}\mid X_i,A_i,Y_i,Z_i]\right\}+\op(n^{-1/2})\\
    &=\frac{1}{ \nev}\sum_{i=1}^{\nev}\E[a_{j,i}\mid X_i, A_i, Y_i,Z_i]+\frac{1}{ \nhs}\sum_{i=1}^{\nhs}\E[a_{i,j}\mid X_i, A_i, Y_i,Z_i]+\op(n^{-1/2})\\ 
    &=\frac{1}{ \nev}\sum_{i=1}^{\nev}\{Y_i-f(X_i,A_i)\}r(X_i)w(X_i,A_i)+\op(n^{-1/2}). 
\end{align*}
From the first line to the second line, we used a U-statistics theory \citep[Chapter 12]{VaartA.W.vander1998As}. From the third line to the fourth line, we used
\begin{align*}\ts
    &\E[a_{j,i}\mid Z_i,X_i,A_i,Y_i] \\
    &=\op(n^{-1/2})+\left\{\frac{Y_i}{u(X_i,A_i)} -\frac{f(X_i,A_i)}{u(X_i,A_i) }\right\}q(X_i)\epol(A_i \mid X_i) \\
    &= \op(n^{-1/2})+Y_ir(X_i)w(A_i,X_i)-r(X_i)f(X_i,A_i)w(A_i,X_i) , \\
    &\E[a_{i,j}\mid Z_i,X_i,A_i,Y_i] = \op(n^{-1/2}).
\end{align*}
This is proved by Theorem \ref{thm:kernel}. Therefore, 
\begin{align*}\ts
    \hat R_{\dm} &=\E_{\nhs}\left[r(X)w(A,X)\left\{Y-f(A,X)\right\}\right]+\E_{\nev}[v(Z)]+\op(n^{-1/2}). 
\end{align*}
The final statement is concluded by CLT.
\end{proof}

\subsection{Proof of Theorem \ref{thm:regret}}
\begin{proof}

We prove the statement following \citet{ZhouZhengyuan2018OMPL}. Though the proof is very similar, for completeness, we sketch the proof the case $\rho=0.5$. Since the estimator is asymptotically linear, the generalization is easy as in Theorem \ref{thm:ipw1}.

Define two scores;
\begin{align*}
\hat \Gamma_i &=    \hat r_w^{(D_i)}(X_i,A_i)\{Y_i-f^{(D_i)}(X_i,A_i) \}+[f(a_1,X_i),\cdots,f(a_{\alpha},X_i)]^{\top},\\
\Gamma_i &= r_w(X_i,A_i)\{Y_i-f(X_i,A_i) \}+[f(a_1,X_i),\cdots,f(a_{\alpha},X_i)]^{\top},
\end{align*}
where $D_i$ is an indicator which cross-fold estimator is used and $\alpha$ is a dimension of the action, $r_{w}(x)=r(x)/\bpol(a\mid x)$. Then, we have $\hat R_{\dml}(\pi)=\frac{2}{n}\{\sum_{i=1}^{n/2}\langle \pi(Z_i), \hat \Gamma_i \rangle \}$. Here, we define the estimator with oracle efficient influence function
$$
\tilde R(\pi)=\frac{2}{n}\sum_{i=1}^{n/2}\langle \pi(Z_i), \Gamma_i \rangle. 
$$
In addition, we define 
\begin{align*}
    \Delta(\pi_a,\pi_b)=R(\pi_a)-R(\pi_b),&\,     \tilde \Delta(\pi_a,\pi_b)=\tilde R(\pi_a)-\tilde R(\pi_b), \\
    \hat \Delta(\pi_a,\pi_b)=\hat R(\pi_a)-\hat R(\pi_b). 
\end{align*}

\paragraph{Step 1:} First, following \citet[Theorem 2]{ZhouZhengyuan2018OMPL}, we prove the following. Let $\tilde \pi\in \argmin_{\pi\in \Pi}\tilde R(\pi)$. Then for any $\delta>0$, with probability at least $1-2\delta$, 
\begin{align*}
    R(\tilde \pi)-R(\pi^{*})\leq O\left(\left\{k(\Pi)+\sqrt{\log(1/\delta)}\right\}\sqrt{\frac{\Upsilon_{*}}{n}}\right),
\end{align*}
where 
\begin{align*}
    \Upsilon_{*} &=\sup_{\pi \in \Pi}\E[\langle \Gamma_i, \pi(Z_i)\rangle^2 ] \\
    &=\sup_{\pi \in \Pi}\E[r(X_i)^2 w_{\pi}^2(X_i,A_i)\{Y_i-f(X_i,A_i)\}^2] +\E[\{v_{\pi}(Z_i)\}^2],
\end{align*}
when $w_{\pi}(a,x)=\pi(a,x)/\bpol(a,x),\,v_{\pi}(x)=\E_{\pi(a\mid x)}[f(a,x)\mid x]$. 

This is proved as follows. We have 
\begin{align*}
     R(\tilde \pi)-R(\pi^{*})\leq \sup_{\pi_a,\pi_b \in \Pi}|\tilde \Delta(\pi_a,\pi_b)-\Delta(\pi_a,\pi_b) |. 
\end{align*}
Then, by using a Chaining argument as Lemma 1 \citep{ZhouZhengyuan2018OMPL}, we can bound an expectation of $\sup_{\pi_a,\pi_b \in \Pi}|\tilde \Delta(\pi_a,\pi_b)-\Delta(\pi_a,\pi_b) |$ via Rademacher complexity. Then, as in Lemma 2 \citep{ZhouZhengyuan2018OMPL}, the high probability bound is obtained via Talagrand inequality. Then, we have
\begin{align*}
    R(\tilde \pi)-R(\pi^{*})\leq O\left(\left\{k(\Pi)+\sqrt{\log(1/\delta)}\right\}\sqrt{\frac{\Upsilon'_{*}}{n}}\right),
\end{align*}
where 
\begin{align*}
    \Upsilon'_{*} &=\sup_{\pi_a,\pi_b \in \Pi}\E[\langle \Gamma_i, \pi_a(Z_i)- \pi_b(Z_i)\rangle^2 ]. 
\end{align*}
This concludes the above statement since  
\begin{align*}
    \Upsilon_{*} &=\sup_{\pi_a,\pi_b \in \Pi}\E[\langle \Gamma_i, \pi_a(Z_i)- \pi_b(Z_i)\rangle^2 ] \\
    &=\sup_{\pi_a,\pi_b\in \Pi}\E[r(X_i)^2\{w_{\pi_a}-w_{\pi_b}\}^2\{Y_i-f(X_i,A_i)\}^2] +\E[\{v_{\pi_a}(Z_i)-v_{\pi_b}(Z_i)\}^2] \\ 
    &\leq  \sup_{\pi \in \Pi}2\E[r(X_i)^2 w_{\pi}^2(X_i,A_i)\{Y_i-f(X_i,A_i)\}^2] +2\E[\{v_{\pi}(Z_i)\}^2]. 
\end{align*}

\paragraph{Step 2:} Assume $\kappa(\Pi)<\infty$, then
\begin{align*}
    \sup_{\pi_a,\pi_b \in \Pi}|\tilde \Delta(\pi_a,\pi_b)- \hat \Delta(\pi_a,\pi_b) |=\op(n^{-1/2}).
\end{align*}

The proof of this statement is based on the double structure of the influence function and cross-fitting. We omit the proof since it is long, and almost the same as Lemma 3  \citep{ZhouZhengyuan2018OMPL}. 

\paragraph{Step 3:}
Finally, based on Theorem 3 \citep{ZhouZhengyuan2018OMPL}, we have 
\begin{align*}
R(\hat \pi)-R(\pi^{*})  &\leq \sup_{\pi_a,\pi_b \in \Pi}|\tilde \Delta(\pi_a,\pi_b)- \hat \Delta(\pi_a,\pi_b) |+ \sup_{\pi_a,\pi_b \in \Pi}|\tilde \Delta(\pi_a,\pi_b)-\Delta(\pi_a,\pi_b) |\\
&\leq O_{p}\left(\left\{k(\Pi)+\sqrt{\log(1/\delta)}\right\}\sqrt{\frac{\Upsilon_{*}}{n}}\right). 
\end{align*}
This means there exists an integer $N_{\delta}$ such that with probability at least $1-2\delta$, for all $n\geq N_{\delta}$:
\begin{align*}
    R(\hat \pi)-R(\pi^{*})\lessapprox \left(k(\Pi)+\sqrt{\log(1/\delta)}\right)\sqrt{\frac{\Upsilon_{*}}{n}}. 
\end{align*}

\begin{remark}
In the general case,
\begin{align*}
    \Upsilon_{*} &=\sup_{\pi \in \Pi}\rho^{-1}\E[r(X_i)^2w_{\pi}^2(X_i,A_i)\{Y_i-f(X_i,A_i)\}^2]+(1-\rho)^{-1}\E[\{v_{\pi}(Z_i)\}^2]. 
\end{align*}
\end{remark}
\end{proof}

\section{OPE with Known Distribution of Evaluation Data}
\label{sec:known}
In this section, we consider a special case where $q(x)$ is known.

By applying \eqref{eq:bound} in Section~\ref{sec:pre}, we obtain the efficiency bound under nonparametric model defined, which is defined as $\tilde{\Upsilon}(\epol)=\mathbb{E}[r^2(X)w^2(A,X)\mathrm{var}[Y \mid A,X]]$.

As the estimator $\hat{R}_{\mathrm{DRCS}}(\epol)$ in Section~\ref{sec:pe_csa}, we construct an estimator with cross-fitting. Instead of \eqref{eq:case2}, we use an estimator defined as $\mathbb{E}_{n^\mathrm{hst}_{k}}[\hat r^{(k)}(X)\hat w^{(k)}(A,X)\{Y-\hat f^{(k)}(A,X) \}]+\mathbb{E}_{q(z)\epol(a\mid z)}[\hat f^{(k)}(a,z)]$. The algorithm is almost the same as before. To estimate $r(x)$, we can simply use density estimation for $p(x)$ since $q(x)$ is known and the integration in $\rE_{q(z)\epol(a\mid z)}[\hat f(a,z)]$ can be taken exactly since $q(x)$ and $\epol(a\mid x)$ are known. Let us denote this estimator as $\tilde{R}_{\mathrm{DRCS}}$. 
We can show that $\tilde{R}_{\mathrm{DRCS}}(\epol)$ achieves the efficiency bound. 
\begin{theorem}[Efficiency of $\tilde{R}_{\mathrm{DRCS}}$]
\label{thm:main1}
For $k\in \{1,\cdots,\xi\}$, assume there exists $p>0,\,q>0,\,p+q\geq 1/2$ such that $\|\hat r^{(k)}(X)\hat w^{(k)}(A,X)-r(X)w(A,X)\|_2=\op(n^{-p})$ and $\|\hat f^{(k)}(A,X)-f(A,X)\|_2=\op(n^{-q})$.
Then, we have $\sqrt{\nhs}(\tilde{R}_{\mathrm{DRCS}}(\epol)-R(\epol))\stackrel{d}{\rightarrow}\bN(0,\tilde{\Upsilon}(\epol))$. 
\end{theorem}
This asymptotic variance is equal to the asymptotic variance when $\rho=0$ as shown in Remark~2 because the case $\rho=0$ implies that we have infinite data from $q(x)$.

\section{Self-Normalized Doubly Robust Estimator with Cross-Fitting}
\label{appdx:self_dml}
In this section, we define self-normalized versions of estimators $\hat{R}_{\dml}(\epol)$ and $\hat R_{\mathrm{IPWCS}}(\epol)$. Let us define the self-normalized DRCS estimator as follows:
\begin{align*}\ts
    \hat{R}_{\mathrm{DRCS}\mathchar`-\mathrm{SN}}(\epol) =&  \frac{1}{\E_{\nhs}\left[ \frac{1}{\hat \pi^\mathrm{b}(A \mid X)}\right]}\E_{\nhs}\left[\hat r(X)\hat w(A,X)\{Y_i-\hat f(A,X)\}\right]\\
    &+\mathbb{E}_{\nev}[\E_{\epol(a \mid Z)}[\hat{f}(a,Z)\mid Z]],
\end{align*}
and the self-normalized IPWCS estimator as follows:
\begin{align*}\ts
    \hat R_{\mathrm{IPWCS}\mathchar`-\mathrm{SN}}(\epol)= \frac{1}{\E_{\nhs}\left[ \frac{1}{\hat \pi^\mathrm{b}(A \mid X)}\right]}\E_{\nhs}\left[\frac{\hat q(X_i)\epol(A \mid X_i)Y}{\hat p(X)\hat \pi^\mathrm{b}(A \mid X_i)}\right]. 
\end{align*}

\section{Algorithm for Off-Policy Learning with Cross-Fitting}
\label{appdx:opl_alg}
In the proposed method of OPL under a covariate shift, we train an evaluation policy by using an estimator $\hat{R}_{\dml}(\epol)$, which is constructed via cross-fitting. In this section, we introduce an algorithm where we use a linear-in-parameter model with kernel functions to approximate a new policy. For $x \in \mathcal{\mathcal{X}}$, a linear-in-parameter model is defined as follows:
\begin{align*}\ts
\label{linear_model}
\pi(a \mid x; \sigma^2) = \frac{\exp(g(a, x; \sigma^2))}{\sum_{a\in\mathcal{A}}\exp(g(a, x; \sigma^2))},
\end{align*}
where $g(a, x; \sigma^2) = \beta^{\top}_{a}\varphi(x; \sigma^2)+\beta_{0,a}$, $\varphi(x; \sigma^2)=\left[\varphi_1(x; \sigma^2), \dots, \varphi_m(x; \sigma^2)\right]^{\top}$, $\varphi_m(x; \sigma^2)$ is the Gaussian kernel defined as 
\begin{align*}\ts
\varphi_u(x; \sigma^2) = \exp\left(-\frac{\|x-c_u\|^2}{2\sigma^2}\right),\,1\leq u \leq m,
\end{align*}
$\{c_1,...,c_m\}$ is $m$ chosen points from $\{X_i\}^{n^{\mathrm{hst}}}_{i=1}$, $\beta_{a}\in\mathbb{R}^m$, and $\beta_{0,a}\in\mathbb{R}$. In optimization, we put a regularization term $\mathcal{R}(\{\beta_{a}, \beta_{0,a}\})$ and train a new policy as follows:
\begin{align*}\ts
      \hat \pi_{\dml} =\argmax_{\pi\in \Pi} \hat{R}_{\dml}(\pi) + \lambda \mathcal{R}(\{\beta_{a}, \beta_{0,a}\}),
\end{align*}
where $\lambda > 0$. The parameters $\sigma^2$ and $\lambda$ are hyper-parameters selected via cross-validation. Thus, in the proposed method, we use the cross-fitting and cross-validation. We describe the algorithm in Algorithm~\ref{alg:opl} with $K$ fold cross-fitting and $L$ fold cross-validation. For brevity, in the algorithm, let us assume $\nhs/\xi, \nhs/L, \nev/\xi, \nev/L \in \mathbb{N}$. In Algorithm~\ref{alg:opl}, we express the objective function with hyper-parameters $\sigma^2$ and $\lambda$ as follows:
\begin{align*}\ts
     \E_{\nhs}\left[\hat r(X)\frac{\pi(A \mid X; \sigma^2)}{\hat \pi^{\mathrm{b}}(A \mid X)}\{Y-\hat f(A,X)\}\right]+ \mathbb{E}_{\nev}[\E_{\epol(a \mid Z)}[\hat f(a,Z) \pi(a \mid Z; \sigma^2)] + \lambda \mathcal{R}(\{\beta_{a}, \beta_{0,a}\}).
\end{align*}

\footnotesize
\begin{algorithm}[tb]
   \caption{Off-policy learning using $\hat{R}_{\dml}(\epol)$ with $\xi$-fold cross-fitting. }
   \label{alg:opl}
\begin{algorithmic}
    \STATE \textbf{Input}: $\xi$: the number of the folds of the cross-fitting for constructing $\hat{R}_{\dml}(\epol)$. $L$: the number of the folds of the cross-validation for constructing the optimal policy. $\Pi$: a hypothesis class of $\epol$. $\{\sigma^2_1,\dots,\sigma^2_{n_{\sigma^2}}\}$: candidates of $\sigma^2$. $\{\lambda_1,\dots,\lambda_{n_\lambda}\}$: candidates of $\lambda$.
    \STATE Take a $\xi$-fold random partition $(I_k)^\xi_{k=1}$ of observation indices $[\nhs] = \{1,\dots,\nhs\}$ such that the size of each fold $I_k$ is $\nhs_k=\nhs/\xi$.
    \STATE Take a $\xi$-fold random partition $(J_k)^\xi_{k=1}$ of observation indices $[\nev] = \{1,\dots,\nev\}$ such that the size of each fold $J_k$ is $\nev_k=\nev/\xi$.
    \STATE For each $k\in[\xi]=\{1,\dots,\xi\}$, define $I^c_k:=\{1,\dots,\nhs\}\setminus I_k$ and $J^c_k:=\{1,\dots,\nev\}\setminus J_k$.
    \STATE Define $\mathcal{S}_k = \{(X_i, A_i, Y_i)\}_{i\in I^c_k}$. 
    \FOR{$k\in[K]$}
    \STATE Construct nuisance estimators $\hat \pi^\mathrm{b}_k(a \mid X)$, $\hat r_k(x)$, and $\hat f_k(a,x)$ using $\mathcal{S}_k$.
    \ENDFOR
    \STATE Take a $L$-fold random partition $(I_\ell)^L_{\ell=1}$ of observation indices $[\nhs] = \{1,\dots,\nhs\}$ such that the size of each fold $I_\ell$ is $\nhs_\ell=\nhs/L$.
    \STATE Take a $L$-fold random partition $(J_\ell)^L_{\ell=1}$ of observation indices $[\nev] = \{1,\dots,\nev\}$ such that the size of each fold $J_\ell$ is $\nev_\ell=\nev/L$.
    \STATE For each $\ell\in[L]=\{1,\dots,L\}$, define $I^c_\ell:=\{1,\dots,\nhs\}\setminus I_\ell$ and $J^c_\ell:=\{1,\dots,\nev\}\setminus J_\ell$.
    \FOR{$\tilde \sigma^2 \in \{\sigma^2_1,\dots,\sigma^2_{n_{\sigma^2}}\}$}
    \FOR{$\tilde \lambda \in \{\lambda_1,\dots,\lambda_{n_\lambda}\}$}
    \STATE Define $Score_{\tilde \sigma^2, \tilde \lambda}=0$.
    \FOR{$\ell\in[L]$}
    \STATE Obtain $\tilde \pi$ by solving the following optimization problem:
    \begin{align*}\ts
	\tilde \pi =& \argmax_{\pi\in \Pi} \E_{\nhs_{I_\ell}}\left[\hat r(X)\frac{\pi(A \mid X; \tilde \sigma^2)}{\hat \pi^{\mathrm{b}}(A \mid X)}\{Y-\hat f(A,X)\}\right]\\
	&+ \E_{\nev_{J_\ell}}[\E_{\epol(a \mid Z)}[\hat f(a,Z) \pi(a \mid Z; \tilde \sigma^2)] + \tilde \lambda \mathcal{R}(\{\beta_{a}, \beta_{0,a}\}),
    \end{align*}
    where $\E_{\nhs_{I_\ell}}$ denotes a empirical approximation using $i\in I_\ell$, $\E_{\nev_{J_\ell}}$ denotes a sample approximation using $j\in J_\ell$, and $\hat \pi^\mathrm{b}$, $\hat r$, and $\hat{f}$ are the corresponding nuisance estimators chosen from $\hat \pi^\mathrm{b}_k$, $\hat r_k$, and $\hat f_k$. 
    \STATE Update the score $Score_{\tilde \sigma^2, \tilde \lambda}$ by 
    \begin{align*}\ts
    &Score_{\tilde \sigma^2, \tilde \lambda}\\
    &= Score_{\tilde \sigma^2, \tilde \lambda} + \E_{\nhs_{I^c_\ell}}\left[\hat r(X)\frac{\tilde \pi(A \mid X; \tilde \sigma^2)}{\hat \pi^{\mathrm{b}}(A \mid X)}\{Y-\hat f(A,X)\}\right] + \E_{\nev_{J_\ell}}[\E_{\epol(a \mid Z)}[\hat f(a,Z) \tilde \pi(a \mid Z; \tilde \sigma^2)],
    \end{align*}
    where $\E_{\nhs_{I^c_\ell}}$ denotes a empirical approximation using $i\in I^c_\ell$, and $\E_{\nev_{J^c_\ell}}$ denotes a sample approximation using $j\in J^c_\ell$.
    \ENDFOR
    \ENDFOR
    \ENDFOR
    \STATE Obtain $\tilde \pi$ by solving the following optimization problem:
    \begin{align*}\ts
	\hat \pi =& \argmax_{\pi\in \Pi} \E_{\nhs}\left[\hat r(X)\frac{\pi(A \mid X; \hat \sigma^2)}{\hat \pi^{\mathrm{b}}(A \mid X)}\{Y-\hat f(A,X)\}\right]\\
	&+ \E_{\nev_{J_\ell}}[\E_{\epol(a \mid Z)}[\hat f(a,Z) \pi(a \mid Z; \hat \sigma^2)] + \hat \lambda \mathcal{R}(\{\beta_{a}, \beta_{0,a}\}),
    \end{align*}
    where $(\hat \sigma^2, \hat \lambda) = \argmax_{(\tilde \sigma^2, \tilde \lambda)\in \{\{\sigma^2_1,\dots,\sigma^2_{n_{\sigma^2}}\}, \{\lambda_1,\dots,\lambda_{n_\lambda}\}\}} Score_{\tilde \sigma^2, \tilde \lambda}$.
\end{algorithmic}
\end{algorithm}
\normalsize

\section{Details of Experiments in Section~\ref{subsec:exp_ope}}
\label{appdx:details_exp}
First, we show the description of the datasets in Table. All datasets are downloaded from \url{https://www.csie.ntu.edu.tw/~cjlin/libsvmtools/datasets/}. 

\begin{table}[t]
\label{Dataset}
\caption{Specification of datasets}
\begin{center}
\scalebox{0.9}[0.9]{\begin{tabular}{cccc}
\hline
Dataset&the number of samples &Dimension &the number of classes\\
\hline
 satimage & 4,435  &  35 & 6 \\
 vehicle & 846  &  18 & 4 \\
 pendigits & 7,496 &  16 & 10\\
\end{tabular}}
\end{center}
\end{table}

In addition to the results shown in Section~\ref{subsec:exp_ope}, we show the performances of IPWCS and DM estimator with nuisance functions estimated by the kernel Ridge regression, which are referred as IPWCS-R and DM-R, and the self-normalized versions of the proposed estimators, DRCS and IPWCS estimator, which are referred as DRCS-SN and IPWCS-SN. In addition, for OPE, we also show the results with the different sample size.

In Tables~\ref{tbl:table1_ape}--\ref{tbl:table2_ape}, we show the additional experimental results with the same setting as Section~\ref{sec:exp}. In this setting, the sample size is fixed at $800$.

In Tables~\ref{tbl:table3_ape}--\ref{tbl:table4_ape}, we show the additional experimental results with $500$ samples. The other setting is the same as Section~\ref{subsec:exp_ope}.

In Tables~\ref{tbl:table5_ape}--\ref{tbl:table6_ape}, we show the additional experimental results with $300$ samples. The other setting is the same as Section~\ref{subsec:exp_ope}.

We also add the OPE and OPL experiment with PenDigits in \cref{tbl:table6,tbl:table7}. In this experiment, the sample size is fixed at $800$.

In Tables~\ref{tbl:table7_ape}--\ref{tbl:table8_ape}, we show the additional experimental results with $1,000$ samples for SatImage and PenDigits datasets. We could not conduct experiments for Vehicle dataset because it only has $800$ samples. The other setting is the same as Section~\ref{subsec:exp_ope}.

For OPE, we highlight in bold the best two estimators in each case. For OPL, we highlight in bold the best one estimator in each case. The proposed DRCS estimator performs well in many datasets. The DM estimator also works well, but the performance dramatically drops when the model is misspecified. 

\begin{table*}[t!]
\begin{center}
\caption{Off-policy evaluation with SatImage dataset with $800$ samples} 
\label{tbl:table1_ape}
\scalebox{0.65}[0.65]{
\begin{tabular}{l|rr|rr|rr|rr|rr|rr|rr}
\hline
\multirow{2}{*}{Behavior Policy} & \multicolumn{2}{c|}{DRCS}  & \multicolumn{2}{c|}{IPWCS} & \multicolumn{2}{c|}{DM} & \multicolumn{2}{c|}{IPWCS-R} &  \multicolumn{2}{c|}{DM-R} & \multicolumn{2}{c|}{DRCS-SN} & \multicolumn{2}{c}{IPWCS-SN}\\
 & MSE & std & MSE & std & MSE & std & MSE & std & MSE & std & MSE & std & MSE & std \\
\hline
$0.7\pi^d+0.3\pi^u$ &  0.107 &  0.032 &  67.448 &  144.845 &  {\bf 0.042} &  0.043 &  {\bf 0.045} &  0.049 &  0.073 &  0.023 &  0.188 &  0.033 &  0.245 &  0.127 \\
$0.4\pi^d+0.6\pi^u$  &  {\bf 0.096} &  0.025 &  74.740 &  155.704 &  0.134 &  0.052 &  {\bf 0.093} &  0.069 &  0.177 &  0.033 &  0.189 &  0.028 &  0.232 &  0.088 \\
$0.0\pi^d+1.0\pi^u$  &  {\bf 0.154} &  0.051 &  58.031 &  103.632 &  0.336 &  0.079 &  {\bf 0.022} &  0.026 &  0.372 &  0.050 &  0.358 &  0.058 &  0.372 &  0.087 \\
\hline
\end{tabular}
}

\caption{Off-policy evaluation with Vehicle dataset with $800$ samples} 
\label{tbl:table2_ape}
\scalebox{0.65}[0.65]{
\begin{tabular}{l|rr|rr|rr|rr|rr|rr|rr}
\hline
\multirow{2}{*}{Behavior Policy} & \multicolumn{2}{c|}{DRCS}  & \multicolumn{2}{c|}{IPWCS} & \multicolumn{2}{c|}{DM} & \multicolumn{2}{c|}{IPWCS-R} &  \multicolumn{2}{c|}{DM-R} & \multicolumn{2}{c|}{DRCS-SN} & \multicolumn{2}{c}{IPWCS-SN}\\
 & MSE & std & MSE & std & MSE & std & MSE & std & MSE & std & MSE & std & MSE & std \\
\hline
$0.7\pi^d+0.3\pi^u$ &  {\bf 0.029} &  0.019 &  218390.000 &  285382.247 &  {\bf 0.038} &  0.035 &  0.568 &  0.319 &  0.040 &  0.014 &  0.086 &  0.019 &  0.099 &  0.044 \\
$0.4\pi^d+0.6\pi^u$ &  {\bf 0.019} &  0.024 &  329825.704 &  454301.175 &  0.095 &  0.062 &  0.576 &  0.357 &  0.089 &  0.019 &  {\bf 0.086} &  0.015 &  0.125 &  0.063 \\
$0.0\pi^d+1.0\pi^u$ &  {\bf 0.037} &  0.030 &  173603.802 &  141163.618 &  0.213 &  0.049 &  0.233 &  0.193 &  0.210 &  0.031 &  {\bf 0.174} &  0.026 &  0.193 &  0.040 \\
\hline
\end{tabular}
}
\end{center}
\end{table*}

\begin{table*}[t!]
\begin{center}
\caption{Off-policy evaluation with SatImage dataset with $500$ samples} 
\label{tbl:table3_ape}
\scalebox{0.65}[0.65]{
\begin{tabular}{l|rr|rr|rr|rr|rr|rr|rr}
\hline
\multirow{2}{*}{Behavior Policy} & \multicolumn{2}{c|}{DRCS}  & \multicolumn{2}{c|}{IPWCS} & \multicolumn{2}{c|}{DM} & \multicolumn{2}{c|}{IPWCS-R} &  \multicolumn{2}{c|}{DM-R} & \multicolumn{2}{c|}{DRCS-SN} & \multicolumn{2}{c}{IPWCS-SN}\\
 & MSE & std & MSE & std & MSE & std & MSE & std & MSE & std & MSE & std & MSE & std \\
\hline
$0.7\pi^d+0.3\pi^u$ &  0.112 &  0.039 &  729.208 &  2433.557 &  {\bf 0.049} &  0.042 &  0.177 &  0.407 &  {\bf 0.079} &  0.033 &  0.181 &  0.041 &  0.269 &  0.120 \\
$0.4\pi^d+0.6\pi^u$  &  {\bf 0.087} &  0.036 &  790.188 &  2139.882 &  0.146 &  0.074 &  {\bf 0.130} &  0.170 &  0.173 &  0.045 &  0.187 &  0.054 &  0.263 &  0.093 \\
$0.0\pi^d+1.0\pi^u$  &  {\bf 0.179} &  0.066 &  453.553 &  1148.372 &  0.335 &  0.097 &  {\bf 0.047} &  0.071 &  0.374 &  0.070 &  0.360 &  0.075 &  0.362 &  0.070 \\
\hline
\end{tabular}
}

\caption{Off-policy evaluation with Vehicle dataset with $500$ samples} 
\label{tbl:table4_ape}
\scalebox{0.65}[0.65]{
\begin{tabular}{l|rr|rr|rr|rr|rr|rr|rr}
\hline
\multirow{2}{*}{Behavior Policy} & \multicolumn{2}{c|}{DRCS}  & \multicolumn{2}{c|}{IPWCS} & \multicolumn{2}{c|}{DM} & \multicolumn{2}{c|}{IPWCS-R} &  \multicolumn{2}{c|}{DM-R} & \multicolumn{2}{c|}{DRCS-SN} & \multicolumn{2}{c}{IPWCS-SN}\\
 & MSE & std & MSE & std & MSE & std & MSE & std & MSE & std & MSE & std & MSE & std \\
\hline
$0.7\pi^d+0.3\pi^u$ &  {\bf 0.029} &  0.020 &  104311.242 &  126027.165 &  {\bf 0.028} &  0.027 &  0.379 &  0.317 &  0.036 &  0.022 &  0.080 &  0.025 &  0.119 &  0.055 \\
$0.4\pi^d+0.6\pi^u$ &  {\bf 0.014} &  0.010 &  186170.520 &  260715.112 & {\bf  0.081} &  0.040 &  0.585 &  0.553 &  0.082 &  0.031 &  0.084 &  0.025 &  0.141 &  0.062 \\
$0.0\pi^d+1.0\pi^u$ & {\bf 0.034} &  0.038 &   82883.403 &  115580.232 &  {\bf 0.149} &  0.064 &  0.230 &  0.235 &  0.184 &  0.042 &  0.160 &  0.035 &  0.212 &  0.044 \\
\hline
\end{tabular}
}
\end{center}
\end{table*}

\begin{table*}[t!]
\begin{center}
\caption{Off-policy evaluation with SatImage dataset with $300$ samples} 
\label{tbl:table5_ape}
\scalebox{0.65}[0.65]{
\begin{tabular}{l|rr|rr|rr|rr|rr|rr|rr}
\hline
\multirow{2}{*}{Behavior Policy} & \multicolumn{2}{c|}{DRCS}  & \multicolumn{2}{c|}{IPWCS} & \multicolumn{2}{c|}{DM} & \multicolumn{2}{c|}{IPWCS-R} &  \multicolumn{2}{c|}{DM-R} & \multicolumn{2}{c|}{DRCS-SN} & \multicolumn{2}{c}{IPWCS-SN}\\
 & MSE & std & MSE & std & MSE & std & MSE & std & MSE & std & MSE & std & MSE & std \\
\hline
$0.7\pi^d+0.3\pi^u$ &  0.103 &  0.043 &   765.985 &   2922.342 &  {\bf 0.026} &  0.027 &  0.125 &  0.115 &  {\bf 0.067} &  0.035 &  0.169 &  0.043 &  0.302 &  0.133 \\
$0.4\pi^d+0.6\pi^u$  &  {\bf 0.074} &  0.051 &    40.273 &     89.381 & {\bf  0.126} &  0.098 &  0.261 &  0.309 &  0.155 &  0.055 &  0.172 &  0.049 &  0.333 &  0.134 \\
$0.0\pi^d+1.0\pi^u$  &  {\bf 0.169} &  0.095 &  4367.009 &  15791.530 &  {\bf 0.297} &  0.084 &  0.375 &  1.293 &  0.341 &  0.073 &  0.324 &  0.060 &  0.333 &  0.105 \\
\hline
\end{tabular}
}

\caption{Off-policy evaluation with Vehicle dataset with $300$ samples} 
\label{tbl:table6_ape}
\scalebox{0.65}[0.65]{
\begin{tabular}{l|rr|rr|rr|rr|rr|rr|rr}
\hline
\multirow{2}{*}{Behavior Policy} & \multicolumn{2}{c|}{DRCS}  & \multicolumn{2}{c|}{IPWCS} & \multicolumn{2}{c|}{DM} & \multicolumn{2}{c|}{IPWCS-R} &  \multicolumn{2}{c|}{DM-R} & \multicolumn{2}{c|}{DRCS-SN} & \multicolumn{2}{c}{IPWCS-SN}\\
 & MSE & std & MSE & std & MSE & std & MSE & std & MSE & std & MSE & std & MSE & std \\
\hline
$0.7\pi^d+0.3\pi^u$ &  {\bf 0.036} &  0.023 &   78064.888 &   80378.226 &  {\bf 0.029} &  0.028 &  0.328 &  0.391 &  0.038 &  0.021 &  0.086 &  0.030 &  0.135 &  0.072 \\
$0.4\pi^d+0.6\pi^u$ &  {\bf 0.020} &  0.020 &  108655.809 &  136013.160 &  0.096 &  0.055 &  0.668 &  0.608 &  {\bf 0.084} &  0.033 &  0.090 &  0.038 &  0.174 &  0.078 \\
$0.0\pi^d+1.0\pi^u$ &  {\bf 0.063} &  0.051 &   59301.622 &   74435.924 &  0.175 &  0.074 &  {\bf 0.125} &  0.161 &  0.204 &  0.053 &  0.173 &  0.043 &  0.216 &  0.064 \\
\hline
\end{tabular}
}
\end{center}
\end{table*}

\begin{table*}[t!]
\begin{center}
\caption{Off-policy evaluation with PenDigits dataset with $800$ samples} 
\label{tbl:table6}
\scalebox{0.65}[0.65]{
\begin{tabular}{l|rr|rr|rr|rr|rr|rr|rr}
\hline
\multirow{2}{*}{Behavior Policy} & \multicolumn{2}{c|}{DRCS}  & \multicolumn{2}{c|}{IPWCS} & \multicolumn{2}{c|}{DM} & \multicolumn{2}{c|}{IPWCS-R} &  \multicolumn{2}{c|}{DM-R} & \multicolumn{2}{c|}{DRCS-SN} & \multicolumn{2}{c}{IPWCS-SN}\\
 & MSE & std & MSE & std & MSE & std & MSE & std & MSE & std & MSE & std & MSE & std \\
\hline
$0.7\pi^d+0.3\pi^u$ &  0.118 &  0.020 &  1074.278 &   838.074 &  {\bf 0.083} &  0.035 &  {\bf 0.052} &  0.045 &  0.089 &  0.014 &  0.237 &  0.035 &  0.174 &  0.065 \\
$0.4\pi^d+0.6\pi^u$ &  {\bf 0.110} &  0.026 &  1328.069 &  1045.287 &  0.220 &  0.053 &  {\bf 0.056} &  0.040 &  0.231 &  0.026 &  0.254 &  0.038 &  0.247 &  0.039 \\
$0.0\pi^d+1.0\pi^u$ &  {\bf 0.314} &  0.086 &   231.043 &   217.068 &  0.503 &  0.049 &  {\bf 0.116} &  0.187 &  0.511 &  0.037 &  0.509 &  0.046 &  0.482 &  0.036 \\
\hline
\end{tabular}
}

\caption{Off-policy learning with PenDigits dataset with $800$ samples} 
\centering
\label{tbl:table7}
\scalebox{0.65}[0.65]{
\begin{tabular}{l|rr|rr|rr}
\hline
\multirow{2}{*}{Behavior Policy} &   \multicolumn{2}{c|}{DRCS}  & \multicolumn{2}{c|}{IPWCS} & \multicolumn{2}{c}{DM}  \\
 &   RWD & STD & RWD & STD & RWD & STD  \\
\hline
$0.7\pi^d+0.3\pi^u$ &  {\bf 0.683} &  0.030 &  0.241 &  0.048 &  0.507 &  0.060 \\
$0.4\pi^d+0.6\pi^u$ &  {\bf 0.678} &  0.039 &  0.252 &  0.032 &  0.445 &  0.096 \\
$0.0\pi^d+1.0\pi^u$ &  {\bf 0.409} &  0.067 &  0.204 &  0.031 &  0.212 &  0.041 \\
\hline
\end{tabular}
}

\end{center}
\end{table*}

\begin{table*}[t!]
\begin{center}
\caption{Off-policy evaluation with SatImage dataset with $1,000$ samples} 
\label{tbl:table7_ape}
\scalebox{0.65}[0.65]{
\begin{tabular}{l|rr|rr|rr|rr|rr|rr|rr}
\hline
\multirow{2}{*}{Behavior Policy} & \multicolumn{2}{c|}{DRCS}  & \multicolumn{2}{c|}{IPWCS} & \multicolumn{2}{c|}{DM} & \multicolumn{2}{c|}{IPWCS-R} &  \multicolumn{2}{c|}{DM-R} & \multicolumn{2}{c|}{DRCS-SN} & \multicolumn{2}{c}{IPWCS-SN}\\
 & MSE & std & MSE & std & MSE & std & MSE & std & MSE & std & MSE & std & MSE & std \\
\hline
$0.7\pi^d+0.3\pi^u$ &  0.111 &  0.024 &   58.724 &   91.964 &  {\bf 0.052} &  0.034 &  {\bf 0.050} &  0.069 &  0.067 &  0.019 &  0.186 &  0.032 &  0.255 &  0.097 \\
$0.4\pi^d+0.6\pi^u$  &  {\bf 0.090} &  0.026 &  118.317 &  188.729 &  0.173 &  0.097 &  {\bf 0.099} &  0.087 &  0.170 &  0.039 &  0.180 &  0.036 &  0.259 &  0.094 \\
$0.0\pi^d+1.0\pi^u$  &  {\bf 0.145} &  0.038 &   82.801 &  103.326 &  0.369 &  0.106 &  {\bf 0.018} &  0.026 &  0.395 &  0.046 &  0.362 &  0.055 &  0.382 &  0.072 \\
\hline
\end{tabular}
}

\caption{Off-policy evaluation with PenDigits dataset with $1,000$ samples} 
\label{tbl:table8_ape}
\scalebox{0.65}[0.65]{
\begin{tabular}{l|rr|rr|rr|rr|rr|rr|rr}
\hline
\multirow{2}{*}{Behavior Policy} & \multicolumn{2}{c|}{DRCS}  & \multicolumn{2}{c|}{IPWCS} & \multicolumn{2}{c|}{DM} & \multicolumn{2}{c|}{IPWCS-R} &  \multicolumn{2}{c|}{DM-R} & \multicolumn{2}{c|}{DRCS-SN} & \multicolumn{2}{c}{IPWCS-SN}\\
 & MSE & std & MSE & std & MSE & std & MSE & std & MSE & std & MSE & std & MSE & std \\
\hline
$0.7\pi^d+0.3\pi^u$ &  0.118 &  0.021 &  1299.936 &   829.752 &  0.094 &  0.029 &  {\bf 0.040} &  0.040 &  {\bf 0.090} &  0.012 &  0.236 &  0.029 &  0.164 &  0.035 \\
$0.4\pi^d+0.6\pi^u$ &  {\bf 0.106} &  0.021 &  1483.730 &  1014.923 &  0.256 &  0.063 &  {\bf 0.067} &  0.078 &  0.241 &  0.023 &  0.262 &  0.028 &  0.255 &  0.043 \\
$0.0\pi^d+1.0\pi^u$ &  {\bf 0.313} &  0.091 &   300.599 &   216.541 &  0.496 &  0.064 &  {\bf 0.099} &  0.167 &  0.531 &  0.033 &  0.523 &  0.025 &  0.496 &  0.033 \\
\hline
\end{tabular}
}
\end{center}
\end{table*}

\end{document}